%% file: neurips_2026_en.tex
\documentclass{article}

\usepackage[preprint]{neurips_2026}

\usepackage{hyperref}       
\usepackage{url}            
\usepackage{booktabs}       
\usepackage{amsmath}
\usepackage{amssymb}
\usepackage{mathtools}
\usepackage{amsthm}
\usepackage{amsfonts}       
\usepackage{nicefrac}       
\usepackage{microtype}      
\usepackage{xcolor}         
\usepackage{subcaption}
\usepackage{booktabs}
\usepackage{multirow}
\usepackage{graphicx}
\usepackage{tikz}
\usetikzlibrary{arrows.meta,positioning,fit,shapes.geometric,calc}
\usepackage{algorithm}

\theoremstyle{plain}
\newtheorem{theorem}{Theorem}[section]
\newtheorem{proposition}[theorem]{Proposition}

\newtheorem{corollary}[theorem]{Corollary}
\theoremstyle{definition}
\newtheorem{definition}[theorem]{Definition}

\theoremstyle{remark}

\title{Design Conditions for Groupwise Learning with Sequence-Level Rewards: Token-Gradient Cancellation}

\author{%
  Fei Ding\thanks{Correspondence to: \texttt{dignfei@gmail.com}.} \\
  Alibaba Group \\
  \And
  Yongkang Zhang \\
  Alibaba Group \\
  \And
  Youwei Wang \\
  Tsinghua University \\
  \And
  Zijian Zeng \\
  Tsinghua University \\
}

\begin{document}

\maketitle

\begin{abstract}
Under sparse terminal rewards, groupwise comparison has become a mainstream paradigm for reinforcement-learning fine-tuning of reasoning models. However, long-horizon training often accumulates ineffective updates, which we call a learning tax, and further induces probability drift among equivalent solutions and entropy collapse. This paper identifies a necessary design condition from the perspective of token-level credit assignment: to prevent reward-irrelevant drift, groupwise objectives must preserve gradient exchangeability at the token-update level, so that gradients on weak-credit and high-frequency tokens cancel within the group. We show that two common mechanisms that break exchangeability make non-cancellation a structural norm rather than an accident. Motivated by this analysis, we introduce minimal within-group transformations that restore, or closely approximate, the cancellation structure in the shared-token subspace. Experiments show that these transformations stabilize training, improve sample efficiency, and enhance final performance, validating the practical value of the proposed design condition.
\end{abstract}

\begin{figure*}[http]
    \centering
    \includegraphics[width=\textwidth]{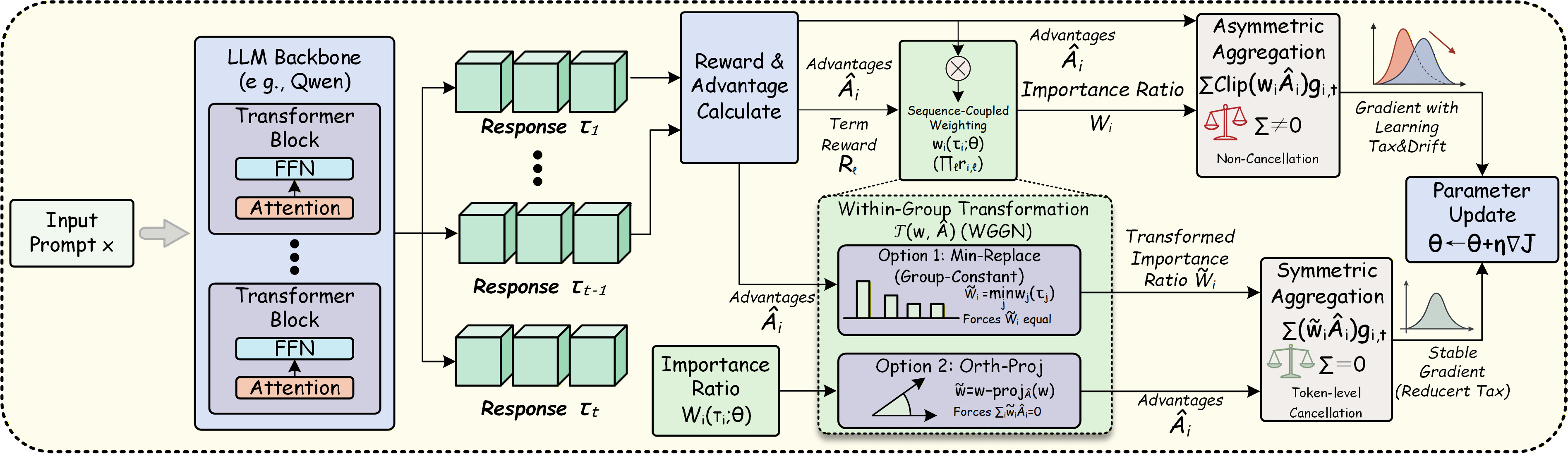}
    \caption{Overview of the DFPO pipeline. After sampling multiple responses for the same input, the method computes group-relative advantages and sequence-coupled importance weights, forming the effective token modulation $W$. It then applies a within-group transformation, Min-Replace or Orth-Proj, to remove within-group asymmetry in $W$, so that shared-token gradients cancel under symmetric aggregation ($\sum=0$). This reduces the learning tax and stabilizes parameter updates. The upper-right comparison illustrates how asymmetric aggregation without the transformation leads to gradient non-cancellation ($\sum\neq 0$) and probability drift.}
    \label{fig:example}
\end{figure*}

\section{Introduction}
\label{sec:introduction}
Under sparse terminal feedback, large language models (LLMs) achieve substantial performance gains on complex reasoning tasks through reinforcement learning. Learning objectives based on within-group comparison have gradually become a dominant paradigm. Their core idea is to compare multiple candidate trajectories sampled for the same input and to learn from their relative relationships within the group. Although these objectives can improve performance markedly in early training, they are difficult to keep stable over long training horizons, where ineffective updates, probability drift among equivalent solutions, and entropy collapse often emerge.

Existing work typically attributes this instability to reward sparsity or optimization noise. These explanations, however, do not answer a more fundamental question: why do different groupwise learning objectives, despite their varied implementation details, repeatedly exhibit similar failure modes? We argue that this phenomenon may arise from a structural limitation rather than from any particular algorithm or hyperparameter choice.

This paper develops a unified view from token-level credit assignment. If a groupwise objective breaks the exchangeability of token updates at the gradient level, especially through sequence-coupled trajectory aggregation or asymmetric segment clipping and selection, then systematic drift, namely a learning tax, together with probability drift and entropy collapse becomes unavoidable. We further show that once this condition is violated, the accumulation of the learning tax and the resulting probability drift are typically predictable.

\noindent\textbf{Contributions.} This paper makes the following contributions:
\begin{itemize}
	\item \textbf{A structural boundary for groupwise learning:} We propose a necessary condition: under sparse terminal feedback, maintaining token-level gradient exchangeability is crucial for stable groupwise learning.
	\item \textbf{A unified gradient perspective:} Through gradient analysis, we clearly distinguish the behavior of exchangeable and non-exchangeable objectives.
	\item \textbf{Constructive validation through structural repair:} We introduce minimal within-group transformations that restore, or closely approximate, gradient cancellation without changing the core framework of groupwise comparison.
\end{itemize}

\section{Related work}
\paragraph{Reinforcement-learning objectives based on within-group comparison.}
Representative methods include GRPO~\cite{shao2024deepseekmathpushinglimitsmathematical}, GSPO~\cite{zheng2025groupsequencepolicyoptimization}, and their variants, such as DAPO~\cite{yu2026dapo}, DCPO~\cite{yang2025dcpodynamicclippingpolicy}, and SSPO~\cite{yang2026ssposubsentencelevelpolicyoptimization}. These methods construct learning signals by comparing multiple trajectories for the same input and have shown advantages on tasks such as mathematical reasoning. Unlike existing work, this paper reveals a structural boundary of groupwise learning objectives from the perspective of token-level credit assignment, providing a unified explanation for failure modes across different groupwise learning methods.

\section{A necessary condition for stable groupwise comparison learning}
\label{sec:proposition}
In groupwise comparison learning, when the objective function breaks the exchangeability of token-level updates in the gradient domain, the intrinsic compensation structure within the group fails, leading to accumulated learning tax and probability drift or entropy collapse. Concretely, we propose the following testable premise: there exists a class of token types that are weakly correlated with the reward, such as generic template tokens, and the learning objective follows a within-group comparison paradigm. If these weakly correlated tokens cannot realize shared-token gradient cancellation within the group, they continually induce nonzero drift, eventually causing a systematic learning tax and entropy collapse.

\subsection{General form: a unified gradient representation of groupwise comparison objectives}
Consider multiple trajectories sampled from a reference policy for the same input. To cover both token-decomposed objectives and sequence-coupled objectives, we write a unified form at the gradient level. Let
\[
	\mathbf g_{i,t}(\theta)
	\triangleq
	\nabla_\theta \log \pi_\theta(y_{i,t}\mid x,y_{i,<t}),
	\qquad
	W_{i,t}(\theta)
	\triangleq
	\alpha_{i,t}\,\omega_{i,t}(\tau_i;\theta).
\]
Here $W_{i,t}$ is the effective token modulation that actually multiplies the gradient of the $t$-th token in the $i$-th trajectory, and $\widehat{A}_i$ denotes the within-group comparison signal. The unified gradient template is
\begin{equation}
	\nabla_\theta\mathcal{J}(\theta) =
	\mathbb{E}_{x}\left[
	\frac{1}{G}\sum_{i=1}^{G} \sum_{t=1}^{T_i}
	\widehat{A}_i\,W_{i,t}(\theta)\,\mathbf g_{i,t}(\theta)
	\right].
	\label{eq:general-objective-new}
\end{equation}
For token-decomposed objectives, one typically has $W_{i,t}=r_{i,t}$. For sequence-coupled objectives with length averaging, $W_{i,t}=s_i(\theta)/T_i$. Therefore, substituting $\omega_{i,t}=s_i(\theta)$ into the unified form does not introduce an additional trajectory-length factor $T_i$.

\subsection{Core definitions: events, cancellation, exchangeability, and learning tax}
\label{sec:core-definitions}

\begin{definition}[Shared context-token event]
\label{def:shared-event}
Fix an input $x$ and a time step $t^\star$. The event $\mathcal{E}_{t^\star}$ means that all trajectories in the group share the same context-token pair $(h^\star, y^\star)$ at this time step; that is, for every $i$, $h_{t^\star}^{(i)}=h^\star$ and $a_{t^\star}^{(i)}=y^\star$.
\end{definition}

\begin{definition}[Within-group cancellation]
\label{def:cancellation}
Under the event $\mathcal{E}_{t^\star}$, we say that this time step achieves \emph{within-group cancellation} if the effective group-aggregated gradient of the shared token is zero:
\begin{equation}
	\frac{1}{G}\sum_{i=1}^{G}
	\widehat{A}_i\,W_{i,t^\star}(\theta)\,\mathbf g_{i,t^\star}(\theta)
	= \mathbf{0}.
	\label{eq:exchangeability_cancellation_direct}
\end{equation}
Here $\mathbf g_{i,t^\star}(\theta)=\nabla_\theta\log\pi_\theta(y^\star\mid h^\star)$. Because all trajectories in $\mathcal{E}_{t^\star}$ share the same $(h^\star,y^\star)$, the above condition is equivalent to
\[
	\frac{1}{G}\sum_{i=1}^{G}
	\widehat{A}_i\,W_{i,t^\star}(\theta)=0.
\]
Intuitively, if token $y^\star$ carries no credit information for distinguishing trajectory quality, its token-level update should be zero.
\end{definition}

\begin{definition}[Shared-token gradient exchangeability]
\label{def:exchangeability}
We say that the shared token under event $\mathcal{E}_{t^\star}$ satisfies \emph{within-group gradient exchangeability} if the non-advantage modulation in its within-group gradient term is exchangeable across the group: there exists a scalar $\bar W_{t^\star}(\theta)$ independent of the trajectory index $i$ such that $W_{i,t^\star}(\theta)=\bar W_{t^\star}(\theta)$ for all $i$. Under the zero-mean constraint $\sum_i\widehat{A}_i=0$, this condition directly implies within-group gradient cancellation for the shared token in \eqref{eq:exchangeability_cancellation_direct}.
\end{definition}

\begin{definition}[Learning tax]
\label{def:learning-tax}
The \emph{learning tax} refers to the following phenomenon. For a class of tokens $\mathcal{C}$ weakly correlated with the terminal reward, satisfying $\mathbb{E}[\widehat{A}\mid y\in\mathcal{C}]\approx 0$, if the within-group gradients of these tokens cannot cancel, then the group-aggregated gradient is strictly nonzero in expectation, as shown in Appendix~\ref{app:tax-lowerbound}. This causes reward-irrelevant parameter drift to accumulate throughout training without bringing a net gain in task performance.
\end{definition}

\subsection{Limitations of statistical cancellation}
In practice, exact token-by-token cancellation is rare when contexts are not identical. However, when the gradient modulation of shared or high-frequency tokens is exchangeable, updates to generic tokens tend to cancel within the group. In contrast, when this modulation is non-exchangeable, the group-aggregated gradient of generic tokens may produce a significant net update. Within-group gradient cancellation therefore remains necessary.

\begin{proposition}[Ineffective updates and distribution drift without within-group cancellation]
	\label{prop:violate_cancel_implies_drift}
	Fix an input $x$ and a time step $t^\star$.
	Consider the event $\mathcal{E}_{t^\star}$ in which all trajectories in the group share the same context-token pair $(h^\star, y^\star)$ at this time step.
	Define the group-aggregated gradient induced by time step $t^\star$ as
	\[
	g_{t^\star}
	\;\triangleq\;
	\frac{1}{G}\sum_{i=1}^{G}
	\widehat{A}_i\,W_{i,t^\star}(\theta)\,\mathbf g_{i,t^\star}(\theta)
	=
	\left(\frac{1}{G}\sum_{i=1}^{G}
	\widehat{A}_i\,W_{i,t^\star}(\theta)\right)\mathbf g^\star(\theta),
	\]
	where $\mathbf g^\star(\theta)=\nabla_\theta\log\pi_\theta(y^\star\mid h^\star)$.
	Let the update be $\theta^+=\theta+\eta g_{t^\star}$, where $\eta>0$ is the step size.
	Define the conditional Fisher information matrix at context $h^\star$ by
	\[
	F_{\theta}(h^\star)
	\;\triangleq\;
	\mathbb{E}_{y\sim \pi_\theta(\cdot\mid h^\star)}
	\Big[
	\nabla_\theta \log \pi_\theta(y\mid h^\star)\;
	\nabla_\theta \log \pi_\theta(y\mid h^\star)^\top
	\Big].
	\]

	If the event $\mathcal{E}_{t^\star}$ does \emph{not} satisfy within-group cancellation, meaning $g_{t^\star}\neq \mathbf{0}$, and if $F_\theta(h^\star)$ is non-degenerate in the direction $g_{t^\star}$:
	\[
	g_{t^\star}^\top F_\theta(h^\star)\, g_{t^\star} \;>\; 0,
	\]
	then this update induces strictly positive conditional distribution drift at context $h^\star$:
	\begin{equation}
		\begin{aligned}
			\mathrm{KL}\!\left(\pi_{\theta^+}(\cdot\mid h^\star)\,\big\|\,\pi_\theta(\cdot\mid h^\star)\right)
			&= \frac{1}{2}\eta^2\, g_{t^\star}^\top F_\theta(h^\star)\, g_{t^\star} \\
			&\hspace{-11.3em}+ O\!(\eta^3\|g_{t^\star}\|^3)
			\quad > 0,\qquad (\eta\ \text{sufficiently small}).
		\end{aligned}
		\label{eq:drift_positive_under_violation}
	\end{equation}

	Therefore, when a shared token carries no credit information for distinguishing trajectory quality, violating within-group cancellation leads to an \emph{ineffective update} of that token's distribution at context $h^\star$, namely reward-irrelevant drift.
\end{proposition}

\begin{proof}[Proof sketch]
	Apply a second-order Taylor expansion to the KL divergence at $\delta\theta=\mathbf{0}$. Its Hessian is the conditional Fisher information matrix $F_\theta(h^\star)$. Substituting $\delta\theta=\eta g_{t^\star}$ gives \eqref{eq:drift_positive_under_violation}. The full proof is given in Appendix~\ref{app:proofs}.
\end{proof}

\begin{corollary}[Accumulated log-odds drift among equivalent solutions $\Rightarrow$ entropy-collapse tendency in the linear region]
	\label{cor:logodds_accum_entropy}
	Fix an input $x$ and consider two semantically equivalent outputs $y_a,y_b$ with the same reward.
	For a class of \emph{sequence-coupled} within-group objectives in a linear update segment, with stop-gradient applied to the effective sequence coefficient,
	\[
	\theta^+=\theta+\eta\,\kappa(y;\theta)\,\mathbf{g}(y;\theta),\qquad \kappa(y;\theta)\ge 0,
	\]
	where $\kappa(y;\theta)\ge 0$ denotes the nonnegative sequence-level effective coefficient associated with output $y$, for example the positive scaling component induced by $\widehat{A}_y\,s_y(\theta)$ under GSPO with stop-gradient, and $\mathbf{g}(y;\theta)\triangleq \nabla_\theta \log \pi_\theta(y\mid x)$ is the score function along output $y$.
	There exists a constant $c(x)>0$ such that
	\begin{equation}
		\Delta \log \frac{\pi_\theta(y_a\mid x)}{\pi_\theta(y_b\mid x)}
		=
		\eta\,c(x)\big(\kappa(y_a;\theta)-\kappa(y_b;\theta)\big)+O(\eta^2).
		\label{eq:cor_logodds_drift}
	\end{equation}
	If $\kappa(y_a;\theta_k)-\kappa(y_b;\theta_k)$ keeps the same sign over $K$ consecutive updates and is lower bounded as $|\kappa(y_a;\theta_k)-\kappa(y_b;\theta_k)|\ge \underline{\Delta}>0$, then the log-odds drift accumulates over $K$ steps, causing $\pi_{\theta_K}(y_a\mid x)/\pi_{\theta_K}(y_b\mid x)\to 0$ or $\infty$. The entropy over the equivalent subset $\{y_a,y_b\}$ thus tends to $0$, giving an entropy-collapse tendency over the set of equivalent solutions.
\end{corollary}

\noindent The proof is provided in Appendix~\ref{app:proofs}.
The minimal algebraic example corresponding to the above claim is provided in Appendix~\ref{sec:toy-problem}.

Two immediate implications follow. The within-group cancellation of a shared token is determined by the effective token modulation $W_{i,t^\star}$. Even when the token ratios at the shared position are identical within the group, a sequence-coupled objective may still induce within-group differences among $W_{1,t^\star},\dots,W_{G,t^\star}$ through differences in the full trajectories. Strict cancellation then requires $\sum_i \widehat A_i W_{i,t^\star}=0$, while the sufficient condition for equal-weight cancellation is $W_{1,t^\star}=\cdots=W_{G,t^\star}$. In continuous parameter spaces, these constraint sets have measure zero; hence non-cancellation is a structural norm.

\subsection{Example: cancellation versus non-cancellation}
\label{sec:min-derivation}

\begin{figure*}[t]
\centering
\begin{tikzpicture}[
  >=Latex, font=\small,
  tok/.style={draw, rounded corners=2pt, minimum width=1.3cm, minimum height=0.6cm, inner sep=2pt},
  shared/.style={tok, fill=blue!10},
  diverge/.style={tok, fill=orange!22},
  coef/.style={font=\footnotesize}
]
\node[font=\bfseries] at (2.0, 2.5) {(a) GRPO (token-decomposed)};
\node at (-0.4, 1.3) {$\tau_1$:};
\node[shared]  at (0.7, 1.3) {answer};
\node[shared]  at (2.0, 1.3) {is};
\node[diverge] at (3.3, 1.3) {122};
\node[coef]    at (4.7, 1.3) {$\widehat A_1=-A$};
\node[coef]    at (0.7, 0.75) {$r$};
\node[coef]    at (2.0, 0.75) {$r$};
\node[coef]    at (3.3, 0.75) {$r_1$};

\node at (-0.4, -0.2) {$\tau_2$:};
\node[shared]  at (0.7, -0.2) {answer};
\node[shared]  at (2.0, -0.2) {is};
\node[diverge] at (3.3, -0.2) {132};
\node[coef]    at (4.7, -0.2) {$\widehat A_2=+A$};
\node[coef]    at (0.7, -0.75) {$r$};
\node[coef]    at (2.0, -0.75) {$r$};
\node[coef]    at (3.3, -0.75) {$r_2$};

\node[align=center] at (2.0, -1.7)
  {Shared-token aggregation: $r\,(\widehat A_1+\widehat A_2)=0$ \\[1pt]
   \textcolor{green!45!black}{\textbf{$\checkmark$ cancellation (no learning tax)}}};

\begin{scope}[xshift=7.5cm]
\node[font=\bfseries] at (2.0, 2.5) {(b) GSPO (sequence-coupled)};
\node at (-0.4, 1.3) {$\tau_1$:};
\node[shared]  at (0.7, 1.3) {answer};
\node[shared]  at (2.0, 1.3) {is};
\node[diverge] at (3.3, 1.3) {122};
\node[coef]    at (4.7, 1.3) {$\widehat A_1=-A$};
\node[coef]    at (0.7, 0.75) {$W_1^\star$};
\node[coef]    at (2.0, 0.75) {$W_1^\star$};
\node[coef]    at (3.3, 0.75) {$W_1^\star$};

\node at (-0.4, -0.2) {$\tau_2$:};
\node[shared]  at (0.7, -0.2) {answer};
\node[shared]  at (2.0, -0.2) {is};
\node[diverge] at (3.3, -0.2) {132};
\node[coef]    at (4.7, -0.2) {$\widehat A_2=+A$};
\node[coef]    at (0.7, -0.75) {$W_2^\star$};
\node[coef]    at (2.0, -0.75) {$W_2^\star$};
\node[coef]    at (3.3, -0.75) {$W_2^\star$};

\node[align=center] at (2.0, -1.7)
  {Shared-token aggregation: $A(W_2^\star-W_1^\star)\neq 0$ \\[1pt]
   \textcolor{red!75!black}{\textbf{$\times$ non-cancellation (learning tax)}}};
\end{scope}
\end{tikzpicture}
\caption{Comparison of cancellation behavior on shared tokens for token-factorized and sequence-coupled objectives. Blue tokens are shared tokens, orange tokens are divergent tokens, and the value below each token is the corresponding effective token modulation $W_i^\star\equiv W_{i,t^\star}$. \textbf{(a) GRPO}: the modulation $r$ is determined only by the token at that position and is independent of the trajectory index; with zero-mean advantages $\sum_i\widehat A_i=0$, the group-aggregated gradient of the shared token cancels exactly. \textbf{(b) GSPO}: the effective modulation $W_i^\star$ is induced by the full sequence weight; differences in the terminal divergent token make $W_1^\star\neq W_2^\star$, so the within-group aggregation is no longer zero even at the shared token, accumulating a learning tax on reward-weak shared or template tokens.}
\label{fig:cancel-vs-noncancel}
\end{figure*}

We compare GRPO and GSPO in Figure~\ref{fig:cancel-vs-noncancel}. Let $G=2$ and suppose the group advantages satisfy $\widehat{A}_1=-\widehat{A}_2\triangleq -A$ with $A>0$. On the shared context-token pair $(h^\star,a^\star)$, assume $r_{1,t^\star}=r_{2,t^\star}=\rho$. The complete derivation is in Appendix~\ref{app:cancel-derivations}.

\paragraph{GRPO (token-decomposed): within-group cancellation.}
The effective token modulation $r_{i,t}$ is determined only by the current token and is independent of the trajectory index. At the shared token, the two trajectories are multiplied by the same scalar $\rho$. Together with $\widehat A_1+\widehat A_2=0$, this makes the group-aggregated gradient of the shared token exactly $\mathbf{0}$.

\paragraph{Asymmetric clipping breaks cancellation.}
Common piecewise operators in practice, such as $\min/\mathrm{clip}$ or threshold-based selection, modify the weight into $\tilde\omega_{i,t}(\theta)=\phi_i(r_{i,t}(\theta))$, where $\phi_i$ explicitly depends on the sign of $\widehat{A}_i$ and therefore forms a \emph{within-group inconsistent} piecewise mapping. Even if $r_{1,t^\star}=r_{2,t^\star}$, zero-mean advantages no longer guarantee cancellation as long as $\phi_1$ and $\phi_2$ induce different local gradient modulations, for example by entering different branches or having different local slopes. Exchangeability is then systematically broken. A full derivation for a typical asymmetric clipping case is given in Appendix~\ref{app:grpo-clip}.

\paragraph{Repair: symmetric clipping restores exchangeability.}
The root cause above is that the piecewise mapping $\phi_i$ explicitly depends on the sign of $\widehat{A}_i$. A minimal repair is a sign-independent symmetric trust-region clipping operator $\phi(r)\triangleq \mathrm{clip}(r,\,1-\varepsilon,\,1+\varepsilon)$, applied identically to all samples regardless of the sign of the advantage. Under a shared context-token event, the gradient modulation for each trajectory in the group is restored to the same scalar $\phi'(r^\star(\theta))r^\star(\theta)$. Combined with the zero-mean constraint $\sum_i\widehat{A}_i=0$, exact cancellation is recovered. We call this correction \emph{GRPO-fix}; see Appendix~\ref{app:sym-grpo-clip}. It is used as a comparison baseline in Section~\ref{sec:experiments}.

\paragraph{GSPO (sequence-coupled): within-group non-cancellation is structural.}
For GSPO, the sequence-level weight is $s_i(\theta)=\big(\pi_\theta(y_i\mid x)/\pi_{\theta_{\mathrm{old}}}(y_i\mid x)\big)^{1/T_i}$. By expanding the log derivative, the effective modulation applied to each token score in a linear segment of GSPO is $s_i(\theta)/T_i$. We therefore write the shared-position modulation as $W_i^\star\triangleq W_{i,t^\star}=s_i(\theta)/T_i$. Because $s_i$ depends on the full trajectory, divergent terminal tokens may make $W_1^\star\neq W_2^\star$ even when the shared token itself is identical. The gradient term aligned with the shared token simplifies to $\tfrac{A}{2}\big(W_2^\star-W_1^\star\big)\nabla_\theta\log\pi_\theta(a^\star\mid h^\star)+\cdots$. Cancellation is equivalent to $W_1^\star=W_2^\star$. This condition holds only on a measure-zero set defined by an exact equality constraint. Hence, under sequence-coupled weights, non-cancellation of shared tokens is a structural norm.

\section{Method: a decoupled group-relative gradient estimator aligned with the structural proposition}
\label{sec:method}

Rather than starting from a particular algorithm, we construct a \emph{general group-relative reinforcement-learning gradient form} to obtain a \textbf{decoupled gradient estimator} that is \emph{strictly aligned} with Proposition~\ref{prop:violate_cancel_implies_drift}. The design principle is simple: we do not redefine the advantage, introduce extra supervision, or change token-level gradient directions. We only remove the structural asymmetric terms identified by the proposition. The overall pipeline is shown in Figure~\ref{fig:example}, and a staged illustration is given in Appendix Figure~\ref{fig:dfpo-method}.

Consider a general group-relative objective
\begin{equation}
	\mathcal{J}(\theta)
	=
	\mathbb{E}_{x,\,\{\tau_i\}_{i=1}^G\sim \pi_{\theta_{\mathrm{old}}}}
	\left[
	\frac{1}{G}\sum_{i=1}^G
	w_i(\tau_i;\theta)\,\widehat{A}_i
	\right],
	\label{eq:method_general_obj}
\end{equation}
with the zero-mean constraint $\sum_i\widehat A_i=0$, where $w_i(\tau_i;\theta)$ is a weight depending on the full trajectory. Let $W_{i,t}\triangleq \alpha_{i,t}w_i(\tau_i;\theta)$ denote the modulation coefficient that actually multiplies the gradient of token $t$. DFPO applies stop-gradient within-group transformations directly to $W_{i,t}$ at shared positions, producing the \emph{decoupled gradient estimator}
\begin{equation}
        \widehat{\nabla_\theta J}_{\mathrm{dec}}
        =
        \mathbb{E}\!\left[
        \frac{1}{G}\sum_{i=1}^G
        \sum_{t=1}^{T_i}
        \widehat{A}_i\,\widetilde W_{i,t}\,
        \nabla_\theta \log \pi_\theta(y_{i,t} \mid x,y_{i,<t})
        \right],
    \label{eq:method_decoupled_estimator}
\end{equation}
where $\alpha_{i,t}\ge 0$ is the coefficient induced by the chosen weight aggregation form, such as $\alpha_{i,t}=1/T_i$ for length normalization. The full log-derivative expansion and the unified gradient template are provided in Appendix~\ref{app:method-derivation}.

\subsection{Two instantiations of the within-group transformation}
\label{sec:method_instances}

We propose two simple and effective transformations. Both aim to suppress asymmetric random modulation in shared-token gradient terms, thereby restoring or approximating within-group gradient cancellation on shared tokens.

\subsubsection{Transformation 1: Min-Replace}
\label{sec:method_minreplace}

Define
\begin{equation}
	W_{\min,t^\star}\triangleq \min_{j\in\{1,\dots,G\}} W_{j,t^\star}(\theta),
	\qquad
	\widetilde W_{i,t^\star} \triangleq W_{\min,t^\star},\ \ \forall i.
	\label{eq:method_minreplace_def}
\end{equation}
This transformation replaces the effective token-gradient modulation coefficient at the shared position within the group by the group minimum, hence the name Min-Replace. It ensures that all trajectories share the same non-advantage modulation scale at that position, removing the gradient non-cancellation contribution caused by within-group differences in $W_{i,t^\star}$. Importantly, this transformation does not make the total gradient zero, since the full gradient is a weighted sum over all token gradients. It only guarantees that cancellation is restored, or closely approximated, in the shared-token subspace where gradient directions are aligned; see Appendix~\ref{app:nonzero-grad}. If an implementation needs to map back to trajectory-level weights and $\alpha_{i,t^\star}>0$, one can set $\widetilde w_i=\widetilde W_{i,t^\star}/\alpha_{i,t^\star}$.

\subsubsection{Transformation 2: advantage-orthogonal reweighting}
\label{sec:method_advorth}

The second transformation does not require the effective token modulations to be equal. Instead, it imposes a \emph{minimal-perturbation reweighting} within the group to suppress systematic bias induced by sequence coupling. For a shared-token position $t^\star$, let
\begin{equation}
	\begin{aligned}
	\mathbf W_{t^\star}
	&\triangleq
	(W_{1,t^\star},\dots,W_{G,t^\star})^\top
	=
	\boldsymbol{\alpha}_{t^\star}\odot \mathbf w,\\
	\widetilde{\mathbf W}_{t^\star}
	&=
	\mathbf W_{t^\star}
	-
	\frac{\widehat{\mathbf A}^{\top}\mathbf W_{t^\star}}
	{\|\widehat{\mathbf A}\|_2^2}\,
	\widehat{\mathbf A}.
	\end{aligned}
	\label{eq:method_orthproj}
\end{equation}
Here $\boldsymbol{\alpha}_{t^\star}=(\alpha_{1,t^\star},\dots,\alpha_{G,t^\star})^\top$, and $\odot$ denotes elementwise multiplication. Length scaling enters only through the projected object $\mathbf W_{t^\star}=\boldsymbol{\alpha}_{t^\star}\odot\mathbf w$; Orth-Proj itself only requires $\widehat{\mathbf A}^{\top}\widetilde{\mathbf W}_{t^\star}=0$. If $\|\widehat{\mathbf A}\|_2=0$, no projection is applied and $\widetilde{\mathbf W}_{t^\star}=\mathbf W_{t^\star}$. If an implementation needs to map back to trajectory-level weights and $\alpha_{i,t^\star}>0$, one can set $\widetilde w_i=\widetilde W_{i,t^\star}/\alpha_{i,t^\star}$. For length-averaged GSPO, $\alpha_{i,t^\star}=1/T_i$. If nonnegative weights must be maintained, Positive Orth-Proj/QP from Appendix~\ref{sec:gvpo_positive_orthproj} can be used.

\paragraph{Alignment with the proposition: why this mitigates learning tax and entropy collapse.}
Proposition~\ref{prop:violate_cancel_implies_drift} shows that the common structural root of learning tax and entropy collapse is token-level asymmetric, non-canceling gradients caused by sequence-coupled weights. Both transformations weaken or eliminate within-group asymmetric modulation in shared-token gradient terms, or its correlation with the advantage vector $\widehat{\mathbf A}$. They therefore restore or approximate within-group gradient cancellation in the shared-token subspace, systematically reducing ineffective updates to weak-reward tokens and suppressing probability drift among equivalent correct solutions.

\paragraph{Theoretical cost: a bias-variance trade-off, not an unbiased estimator.}
We emphasize that DFPO is a \emph{structural stabilization correction}, not an unbiased estimator of the original importance-sampling objective. The within-group transformations break the Radon-Nikodym form of the original IS weights and thus introduce bias. Under the stop-gradient convention, however, this bias does not flip the update direction of any token: each trajectory's update sign is still determined by $\widehat{A}_i$. Its effective role is to apply conservative proportional shrinkage to large-ratio or tail samples within the group, substantially reducing the variance of asymmetric modulation and alleviating the negative feedback loop in which clipping or KL constraints are over-triggered and the signal is then removed. When training remains in a small trust region, so that the relative dispersion of the post-clipping effective modulation $\bar W_{i,t^\star}$ at shared positions is small, the bias upper bound approaches zero with that dispersion; see Appendix~\ref{sec:minreplace_no_reverse}. In short, DFPO accepts controlled bias in exchange for systematically reducing shared-token drift.

\subsection{Testable predictions: lower learning tax leads to better final performance and more stable convergence}
\label{sec:predictions}

Based on the structural repair, we propose three predictions that can be tested on HMMT25, AIME25, and LiveCodeBench. (i) \textbf{Compute efficiency}: under matched compute, the number of steps needed to reach a fixed performance threshold, $\mathrm{Steps}(\mathrm{Score}\ge\kappa)$, should decrease. (ii) \textbf{Convergence stability}: the second-difference jitter of the training curve, $\mathrm{Jitter}_2(m)\triangleq\tfrac{1}{T-2}\sum_{t}|m_{t+2}-2m_{t+1}+m_t|$, should be lower than that of the baseline. (iii) \textbf{Final performance}: reducing reward-irrelevant drift should yield higher final evaluation metrics. The metric definitions and compute-matched setup are detailed in Appendix~\ref{app:predictions-details}.

\section{Experiments}
\label{sec:experiments}

This section validates the three predictions in Section~\ref{sec:predictions} under a compute-matched protocol. We compare baseline methods with our method, using two within-group transformations: Min-Replace and Orth-Proj. For fairness, all methods are matched to the same training compute budget, with the same total number of generated tokens and the same number of parameter-update steps. Training and evaluation settings are identical. The DFPO algorithm, Drift Fixing Policy Optimization implemented through our within-group transformations, and its hyperparameters are detailed in Appendix~\ref{sec:gvpo}.

\textbf{Tasks and datasets.} We evaluate the proposed method on mathematical and coding reasoning benchmarks, including \textit{HMMT25}~\cite{balunovic2026matharena}, \textit{AIME25}~\cite{maa_aime2025}, and \textit{LiveCodeBench v6 (25.02-25.05)}~\cite{jain2025livecodebench}.

\textbf{Base models.} Qwen3-32B and Qwen3-Next-80B-A3B-Thinking are both from the Qwen3 family~\cite{yang2025qwen3technicalreport}.

\textbf{Baselines and comparison setup.} We compare against (1) GSPO, (2) GRPO, and (3) GRPO-fix, which repairs the asymmetric clipping in GRPO according to our design principle. Algorithmic details are provided in Appendix~\ref{app:sym-grpo-clip}, and experimental parameter settings are given in Appendix~\ref{app:impl-details}.

\paragraph{Compute-matched protocol.} We match total training compute by enforcing the following constraints: (i) the same total number of generated tokens, (ii) the same number of optimizer parameter-update steps, and (iii) the same rollout budget per prompt. All methods use the same model, decoding strategy, batching scheme, and hardware configuration.

\begin{table*}[t]
	\caption{Results on Qwen3-32B, Qwen3-Next-80B-A3B-Thinking, and GPT-OSS-120B. We report the mean over five random seeds and the corresponding 95\% bootstrap confidence interval (mean $\pm$ 95\% CI). Under a paired bootstrap test, the improvements over the baselines are statistically significant ($p<0.01$). Inference settings are provided in Appendix~\ref{app:inference-settings}.}
	\label{tab:scaled_result_qwen3}
	\centering
	\setlength{\tabcolsep}{3pt}
	\renewcommand{\arraystretch}{1.10}
	\resizebox{\textwidth}{!}{
	\begin{tabular}{lcccccccccc}
		\toprule
		\multirow{2}{*}{\textbf{Method}} &
		\multicolumn{3}{c}{\textbf{Qwen3-32B Acc avg@32 (\%)}} &
		\multicolumn{3}{c}{\textbf{Qwen3-Next Acc avg@32 (\%)}} &
		\multicolumn{3}{c}{\textbf{GPT-OSS-120B Acc avg@32 (\%)}} \\
		\cmidrule(lr){2-4}\cmidrule(lr){5-7}\cmidrule(lr){8-10}
		& \textbf{AIME25} & \textbf{LiveCodeBench} & \textbf{HMMT25}
		& \textbf{AIME25} & \textbf{LiveCodeBench} & \textbf{HMMT25}
		& \textbf{AIME25} & \textbf{LiveCodeBench} & \textbf{HMMT25} \\
		\midrule
		DFPO (Min-Replace) & 82.5$\pm$1.1 & 71.6$\pm$0.7 & \textbf{61.4}$\pm$1.5 & \textbf{93.2}$\pm$0.9 & 75.1$\pm$0.8 & 80.1$\pm$1.1 & \textbf{98.3}$\pm$0.5 & 84.7$\pm$0.6 & \textbf{91.9}$\pm$0.9 \\
		DFPO (Orth-Proj)   & \textbf{82.6} & \textbf{71.6} & 61.3 & 93.1 & \textbf{75.2} & \textbf{80.2} & 98.2 & \textbf{84.8} & 91.7 \\
		\midrule
		GSPO               & 76.9$\pm$1.3 & 64.7$\pm$1.5 & 55.8$\pm$0.9 & 89.8$\pm$1.7 & 71.0$\pm$1.4 & 75.8$\pm$1.2 & 96.5$\pm$0.5 & 80.3$\pm$1.2 & 87.6$\pm$1.0 \\
		GRPO               & 76.9$\pm$1.2 & 64.5$\pm$0.8 & 55.5$\pm$1.5 & 89.7$\pm$1.2 & 70.9$\pm$0.9 & 75.5$\pm$1.8 & 96.3$\pm$0.7 & 80.1$\pm$0.8 & 87.4$\pm$1.5 \\
		GRPO-fix           & 80.6$\pm$1.4 & 69.1$\pm$1.2 & 59.6$\pm$0.9 & 91.9$\pm$1.5 & 73.9$\pm$1.1 & 79.4$\pm$0.7 & 98.0$\pm$0.6 & 83.6$\pm$0.9 & 90.8$\pm$0.6 \\
		\bottomrule
	\end{tabular}}
\end{table*}

\section{Results and analysis}

\begin{figure}
	\centering
	\includegraphics[width=0.9\linewidth]{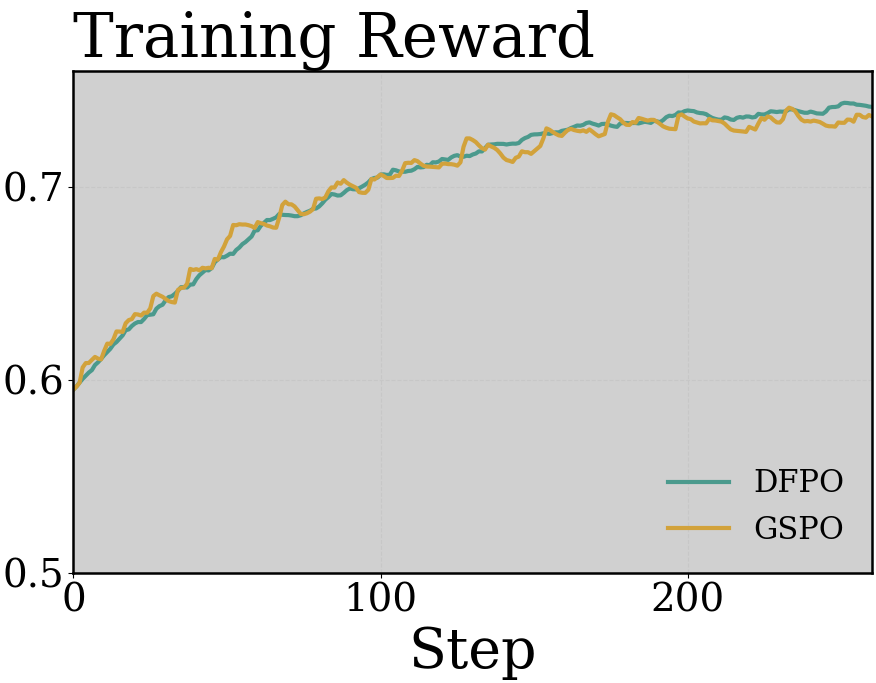}
	\caption{Training curves on Qwen3-Next-80B-A3B-Thinking. Under the \textbf{compute-matched} setting, \textbf{DFPO} achieves substantially higher training efficiency than GSPO.}

	\label{fig:quxiaotu}
\end{figure}

As shown in Figure~\ref{fig:quxiaotu}, under the \textbf{compute-matched} setting, DFPO needs only 91\% of the compute used by GSPO to reach the training reward threshold of 0.70, validating Prediction 1. Its training-return curve is also smoother, with second-difference jitter satisfying $\mathrm{Jitter}_2(m^{\mathrm{DFPO}})<\mathrm{Jitter}_2(m^{\mathrm{GSPO}})$, validating Prediction 2. Table~\ref{tab:scaled_result_qwen3} shows that DFPO achieves higher final performance on AIME25, LiveCodeBench, and HMMT25, validating Prediction 3.

\subsection{Mechanism validation: sequence-coupled breakdown and decoupled recovery}
\label{sec:mechanism}

We report two mechanism metrics: (i) the asymmetry of within-group modulation, $\mathrm{Asym}(t)\triangleq\mathrm{Var}_{i}(W_{i,t}(\theta)\widehat A_i)$, where larger values indicate that shared tokens are harder to cancel within the group; and (ii) the update-energy ratio over a high-frequency token bucket $B$, $\mathrm{Energy}(B)\triangleq \sum_{t\in B}\|\nabla_\theta\ell_t\|_2/\sum_t\|\nabla_\theta\ell_t\|_2$, which reflects the proportion of reward-irrelevant updates. In our experiments, the proposed method substantially reduces $\mathrm{Asym}$ and lowers $\mathrm{Energy}$ in high-frequency buckets, indicating that the learning tax is structurally reduced. Complete metric definitions and the sampling protocol are provided in Appendix~\ref{app:mechanism-metrics}.

\subsection{Ablation study}
\begin{table}[ht]
	\caption{Ablation results on Qwen3-32B.}
	\label{tab:ablation}
	\centering
	\renewcommand{\arraystretch}{1.1}
	\resizebox{0.50\textwidth}{!}{
		\begin{tabular}{lcccc}
			\toprule
			\textbf{Model variant} & \textbf{AIME25 (\%)} & \textbf{LiveCodeBench (\%)} & \textbf{HMMT25 (\%)} &  \\
			\midrule
			\textbf{DFPO(Min-Replace)}    & \textbf{82.5} & \textbf{71.6} & \textbf{61.4} &  \\
			GSPO (baseline)               & 76.9 & 64.7 & 55.8  \\
			DFPO (no stop-grad)           & 75.6 & 63.8 & 54.9  \\
			DFPO (scale by 0.5)           & 75.1 & 63.2 & 54.1   \\
			\bottomrule
		\end{tabular}
	}
\end{table}
Table~\ref{tab:ablation} shows that removing stop-gradient or replacing the within-group correction with global scaling substantially degrades performance, confirming that the gains come from restoring within-group cancellation rather than merely shrinking the update; see Appendix~\ref{app:ablation} for discussion.

\section{Conclusion}
\label{sec:conclusion}

This paper formalizes a recurring form of training instability as a \emph{structural boundary}: under sparse terminal rewards, the stability of group-based reinforcement learning is constrained by the \emph{token-level exchangeability} of the objective. Through gradient decomposition and minimal counterexamples, we show that sequence-level multiplicative coupling breaks this symmetry, causing gradient cancellation to fail and inducing accumulated learning tax and entropy collapse. We argue that restoring this symmetry is a \emph{necessary but not sufficient} condition for mitigating these issues. Building on this insight, we introduce two minimal within-group transformations that restore the cancellation structure on shared tokens. Experiments validate our testable predictions and demonstrate the explanatory and practical value of this structural boundary.

\section{Limitations}
\label{sec:limitations}

(1) This paper characterizes a structural \emph{necessary condition}. With terminal rewards alone, credit assignment remains unidentifiable, so the proposed repair can only mitigate or delay instability rather than guarantee its complete removal.
(2) Our key derivations reveal the structural difference between decomposed and coupled objectives. Their interactions with mechanisms such as clipping, pruning, and normalization require further analysis.
(3) Projection-based transformations may introduce bias. We use them as minimal viable constructions for validating the proposition, but we do not fully explore the space of optimal implementations or a broader set of baselines.

\section*{Broader impact}

GRPO and its group-relative reinforcement-learning variants have become important techniques for post-training reasoning-oriented large language models, and many large-model development teams use them to improve mathematical, coding, and long-chain reasoning capabilities. Therefore, the structural issue identified in this paper affects not only a single algorithmic detail, but potentially a class of LLM training paradigms that are now widely used. If groupwise objectives continually produce non-canceling ineffective updates on shared or high-frequency tokens, training can consume large numbers of expensive sampling tokens, GPU hours, and optimization steps without producing commensurate capability gains. This work may substantially improve sample efficiency, reduce training cost, shorten iteration cycles, and improve the stability and reproducibility of training outcomes. The direction therefore has clear economic and engineering value, especially in industrial settings that rely on GRPO-like algorithms for reasoning-model post-training.

\newpage
\appendix

\section{Ablation study}
\label{app:ablation}

Table~\ref{tab:ablation} compares DFPO with two variants: DFPO (no stop-grad), which does not freeze the within-group transformation during backpropagation, and DFPO (scale by 0.5), which changes Eq.~\eqref{eq:gvpo_min_on_bar} to $\widetilde W_{i,t^\star}(\theta)\triangleq 0.5\,\bar W_{i,t^\star}(\theta)$. Both variants are substantially weaker than Min-Replace on all three benchmarks. First, removing stop-gradient consistently degrades performance, showing that the within-group transformation should act as a constant coefficient. Second, global scaling by 0.5 rules out the alternative explanation that the gain merely comes from a more conservative step size: the improvement arises from the within-group correction of the shared-token gradient cancellation structure, not from simply shrinking the update. This result is consistent with our mechanism analysis of the exchangeability-cancellation necessary condition.

\section{Proofs of Proposition 3.1 and Corollary 3.2}
\label{app:proofs}

\paragraph{Proof of Proposition~\ref{prop:violate_cancel_implies_drift}.}
Let $\delta\theta\triangleq \theta^+-\theta=\eta g_{t^\star}$. For a fixed context $h^\star$, consider the function
\[
\varphi(\delta\theta)\triangleq
\mathrm{KL}\!\left(\pi_{\theta+\delta\theta}(\cdot\mid h^\star)\,\big\|\,\pi_{\theta}(\cdot\mid h^\star)\right).
\]
We have $\varphi(\mathbf{0})=0$. Since the KL divergence is minimized at identical distributions, $\nabla_{\delta\theta}\varphi(\mathbf{0})=\mathbf{0}$. A second-order Taylor expansion at $\delta\theta=\mathbf{0}$ gives
\begin{equation}
	\varphi(\delta\theta)
	=
	\frac{1}{2}\,\delta\theta^\top
	\left.\nabla_{\delta\theta}^2 \varphi(\delta\theta)\right|_{\delta\theta=\mathbf{0}}
	\delta\theta
	+ O(\|\delta\theta\|^3).
\end{equation}
A standard result states that the Hessian at $\delta\theta=\mathbf{0}$ equals the conditional Fisher information matrix:
\[
\left.\nabla_{\delta\theta}^2 \varphi(\delta\theta)\right|_{\delta\theta=\mathbf{0}}
= F_{\theta}(h^\star).
\]
Substituting $\delta\theta=\eta g_{t^\star}$ yields the second-order expansion in \eqref{eq:drift_positive_under_violation}. When $g_{t^\star}\neq \mathbf{0}$ and $g_{t^\star}^\top F_\theta(h^\star)g_{t^\star}>0$, the leading term is strictly positive. Choosing sufficiently small $\eta$ ensures that the third-order remainder does not change the sign, giving a strictly positive KL divergence and hence an ineffective update. \qed

\paragraph{Proof of Corollary~\ref{cor:logodds_accum_entropy}.}
For a fixed $y$, a first-order Taylor expansion gives $\Delta \log \pi_\theta(y\mid x)=\eta\,c(x)\kappa(y;\theta)+O(\eta^2)$, where $c(x)>0$ absorbs normalization constants. Subtracting the two expressions yields \eqref{eq:cor_logodds_drift}. Summing over $k=0,\dots,K-1$ gives linear accumulation of the log odds. As the absolute value of this quantity diverges, the normalized binary probability $p_k=\frac{\pi_{\theta_k}(y_a\mid x)}{\pi_{\theta_k}(y_a\mid x)+\pi_{\theta_k}(y_b\mid x)}$ satisfies $p_k\to 0$ or $1$, so the binary entropy $h(p_k)\to 0$. \qed


\section{Full derivation for the example in Section 3.4: GRPO cancellation and GSPO non-cancellation}
\label{app:cancel-derivations}

To highlight the structural difference, we analyze the objectives in a linear region, ignoring the piecewise effects of $\min/\mathrm{clip}$. This local analysis does not affect the conclusion about gradient decomposability versus coupling. Consider a group of size $G=2$ with trajectories $\tau_1,\tau_2$ for the same input $x$, and suppose the group advantages satisfy
\begin{equation}
	\widehat{A}_1 = -\widehat{A}_2 \triangleq -A, \qquad A>0.
	\label{eq:two-adv}
\end{equation}
Define the token-level importance ratio as
\begin{equation}
	r_{i,t}(\theta) = \frac{\pi_\theta(y_{i,t} \mid h_t^{(i)})}{\pi_{\theta_{\text{old}}}(y_{i,t} \mid h_t^{(i)})}.
	\label{eq:ratio_ht}
\end{equation}

\paragraph{GRPO (token-decomposed form).}
In the token-decomposed linear segment of GRPO, the objective can be written as
\begin{equation}
	\mathcal{J}_{\text{tok}}(\theta) = \frac{1}{2} \sum_{i=1}^2 \sum_{t} r_{i,t}(\theta) \widehat{A}_i.
	\label{eq:Jtok}
\end{equation}
For a shared context-token pair $(h^\star,a^\star)$, suppose that at time step $t^\star$, the two trajectories satisfy $h_{t^\star}^{(1)} = h_{t^\star}^{(2)} = h^\star$ and $a_{t^\star}^{(1)} = a_{t^\star}^{(2)} = a^\star$. In a local neighborhood, $r_{1,t^\star}(\theta) = r_{2,t^\star}(\theta) = \rho$. The gradient contribution of this token pair is
\begin{align}
	\nabla_\theta \mathcal{J}_{\text{tok}}^{(t^\star)}
	&= \frac{1}{2} \left( \widehat{A}_1 \nabla_\theta r_{1,t^\star}(\theta) + \widehat{A}_2 \nabla_\theta r_{2,t^\star}(\theta) \right) \nonumber \\
	&= \frac{1}{2} \left( \widehat{A}_1 \rho + \widehat{A}_2 \rho \right) \nabla_\theta \log \pi_\theta(a^\star \mid h^\star) \nonumber \\
	&= \frac{\rho}{2} (\widehat{A}_1 + \widehat{A}_2) \nabla_\theta \log \pi_\theta(a^\star \mid h^\star) = \mathbf{0},
\end{align}
where the last equality uses $\widehat{A}_1 + \widehat{A}_2 = 0$. Thus, \emph{the gradients of the shared context-token cancel within the group}.

\paragraph{Asymmetric clipping breaks cancellation.}
The exact cancellation above relies on a key assumption: the shared token is multiplied by the \emph{same effective coefficient} within the group. In practice, many methods introduce piecewise operators, such as $\min/\mathrm{clip}$ or threshold-based selection, that modify each trajectory's effective weight to
\begin{equation}
	\tilde{\omega}_{i,t}(\theta) = \phi_i\left( r_{i,t}(\theta) \right),
	\label{eq:piecewise_weight}
\end{equation}
where $\phi_i(\cdot)$ is a \emph{within-group inconsistent} piecewise mapping induced by the sign of the advantage, threshold triggers, or implementation details. Even if $r_{1,t^\star}=r_{2,t^\star}$ under a shared context-token event, $\widehat{A}_1+\widehat{A}_2=0$ no longer guarantees cancellation if $\phi_1$ and $\phi_2$ induce different local gradient modulations or trigger different segments. Thus, \emph{asymmetric piecewise clipping can systematically break exchangeability and cancellation even under token-decomposed structure}.

\paragraph{GSPO (sequence-coupled form).}
In the sequence-coupled linear segment of GSPO, write the effective token modulation at the shared time step $t^\star$ as
\[
	W_i^\star \triangleq W_{i,t^\star}.
\]
The gradient term aligned with $\nabla_\theta\log\pi_\theta(a^\star\mid h^\star)$ is
\begin{align}
	\nabla_\theta \mathcal{J}_{\text{seq}}^{(t^\star)}
	&=
	\frac{1}{2}\left(
	\widehat{A}_1 W_1^\star+\widehat{A}_2 W_2^\star
	\right)\nabla_\theta\log\pi_\theta(a^\star\mid h^\star)
	+\cdots \nonumber\\
	&=
	\frac{1}{2}\left(
	-A\cdot W_1^\star
	+A\cdot W_2^\star
	\right)\nabla_\theta\log\pi_\theta(a^\star\mid h^\star)
	+\cdots \nonumber\\
	&=
	\frac{A}{2}\left(W_2^\star-W_1^\star\right)\nabla_\theta\log\pi_\theta(a^\star\mid h^\star)
	+\cdots.
	\label{eq:seq-noncancel}
\end{align}
As long as $W_1^\star \neq W_2^\star$, this term is strictly nonzero, meaning that \emph{the shared context-token gradient cannot cancel within the group}. Equation~\eqref{eq:seq-noncancel} shows that cancellation holds only on the degenerate set $W_1^\star=W_2^\star$, which is defined by an exact equality constraint and has measure zero in a continuous parameter space.


\section{Complete definitions of the testable predictions and mechanism metrics in Section 4.5}
\label{app:predictions-details}

\paragraph{Prediction 1 (compute efficiency).}
If the within-group transformation effectively suppresses ineffective updates and removes interference, learning efficiency should improve. Under matched compute, the method should reach a fixed performance threshold faster:
\begin{equation}
	\text{Steps}\big(\mathrm{Score}\ge \kappa\big)\ \downarrow,
	\label{eq:pred_threshold}
\end{equation}
where $\mathrm{Score}$ is the evaluation metric for the corresponding benchmark.

\paragraph{Prediction 2 (convergence stability).}
If the within-group transformation focuses the update on effective gradients, local oscillations in the training curve should decrease. We use second-difference jitter:
\begin{equation}
	\mathrm{Jitter}_2(m)\triangleq \frac{1}{T-2}\sum_{t=1}^{T-2}\big|m_{t+2}-2m_{t+1}+m_t\big|,
	\label{eq:pred_jitter2}
\end{equation}
and predict under matched compute that
\begin{equation}
	\mathrm{Jitter}_2\!\big(m^{\mathrm{ours}}\big)
	\;<\;
	\mathrm{Jitter}_2\!\big(m^{\mathrm{base}}\big).
	\label{eq:pred_stability_jitter2}
\end{equation}

\paragraph{Prediction 3 (final performance).}
Ineffective updates repeatedly bias high-frequency tokens or templates that are unrelated to the reward, continually reshaping the surface distribution, namely inducing distribution drift. This causes parameter drift and performance regression on unoptimized tasks or subdistributions, a form of catastrophic forgetting. Reducing such ineffective updates should lead to higher final performance:
\begin{equation}
	\mathrm{Score}_{\mathrm{final}}\ \uparrow.
	\label{eq:pred_finalscore}
\end{equation}

\section{Complete definitions of mechanism metrics}
\label{app:mechanism-metrics}

\paragraph{Asymmetry of gradient modulation.}
\begin{equation}
	\mathrm{Asym}(t)
	=
	\mathrm{Var}_{i\in\{1,\dots,G\}}
	\!\left(
	W_{i,t}(\theta)\,\widehat{A}_i
	\right),
	\label{eq:asymmetry}
\end{equation}
where $W_{i,t}\widehat{A}_i$ is the effective modulation of the shared-token gradient term. Larger $\mathrm{Asym}$ indicates that shared or similar tokens are less likely to cancel within the group. The decoupled transformation substantially reduces $\mathrm{Asym}$ in experiments.

\paragraph{Energy on high-frequency tokens.}
\begin{equation}
	\mathrm{Energy}(B)
	=
	\frac{\sum_{t\in B}\|\nabla_\theta \ell_t\|_2}
	{\sum_{t}\|\nabla_\theta \ell_t\|_2},
	\label{eq:energy}
\end{equation}
where $B$ is a token bucket defined by frequency. Compared with the baseline, our method lowers $\mathrm{Energy}$ in high-frequency buckets, indicating a reduction in reward-irrelevant updates, or learning tax.


\section{Full derivation of the decoupled gradient estimator}
\label{app:method-derivation}

\paragraph{Unified gradient template.}
Applying the log-derivative trick to the objective in \eqref{eq:method_general_obj} gives
\begin{equation}
	\begin{aligned}
		\nabla_\theta \mathcal{J}(\theta)
		= \;&
		\mathbb{E}\!\left[
		\frac{1}{G}\sum_{i=1}^G
		\sum_{t=1}^{T_i}
		\widehat{A}_i\,W_{i,t}(\theta)\,\mathbf{g}_{i,t}(\theta)
		\right], \\
		\mathbf{g}_{i,t}(\theta)
		\triangleq \;&
		\nabla_\theta \log \pi_\theta\!\big(y_{i,t} \mid x,y_{i,<t}\big),
	\end{aligned}
	\label{eq:method_grad_template}
\end{equation}
where $W_{i,t}(\theta)\triangleq \alpha_{i,t}w_i(\tau_i;\theta)$, and $\alpha_{i,t}\ge 0$ is the coefficient induced by the chosen weight aggregation form, for example $\alpha_{i,t}=1/T_i$ for length normalization.

\paragraph{Design principle.}
Based on Proposition~\ref{prop:violate_cancel_implies_drift}, the design principle of DFPO is:
\begin{quote}
	\emph{We do not try to change token-level gradient directions, redefine advantages, or introduce additional supervision. We only remove the structural asymmetric terms introduced by sequence-coupled effective modulation, which correspond directly to the proposition.}
\end{quote}
Accordingly, we focus on the modulation coefficient $W_{i,t}=\alpha_{i,t}w_i(\tau_i;\theta)$ that actually multiplies the token gradient, and apply deterministic transformations only \emph{within the group} to weaken or eliminate its ability to break token-level gradient symmetry. Replacing $W_{i,t}$ in the unified template~\eqref{eq:method_grad_template} by stop-gradient $\widetilde W_{i,t}$ yields Eq.~\eqref{eq:method_decoupled_estimator} in the main text. Orth-Proj directly uses the effective token-modulation vector $\mathbf W_{t^\star}$ at the shared position and projects it onto the subspace orthogonal to $\widehat{\mathbf A}$. For length-averaged GSPO, $\alpha_{i,t^\star}=1/T_i$ is already included in $\mathbf W_{t^\star}$; other implementations use the $\alpha_{i,t^\star}$ corresponding to their own token aggregation form.


\section{DFPO method diagram}
\label{app:dfpo-method-fig}

\begin{figure}[H]
\centering
\resizebox{0.98\textwidth}{!}{
\begin{tikzpicture}[
  >=Latex,
  font=\footnotesize,
  node distance=0.62cm and 0.55cm,
  stage/.style={draw, rounded corners=3pt, text width=0.26\textwidth, minimum height=1.05cm,
    align=center, fill=gray!7, inner sep=4pt},
  op/.style={draw, rounded corners=3pt, text width=0.26\textwidth, minimum height=1.05cm,
    align=center, fill=blue!7, inner sep=4pt},
  good/.style={draw, rounded corners=3pt, text width=0.34\textwidth, minimum height=1.05cm,
    align=center, fill=green!8, inner sep=4pt},
  warn/.style={draw, rounded corners=3pt, text width=0.34\textwidth, minimum height=1.05cm,
    align=center, fill=red!7, inner sep=4pt},
  method/.style={draw, rounded corners=2pt, text width=0.21\textwidth, minimum height=0.75cm,
    align=center, fill=white, inner sep=3pt, font=\scriptsize},
  arr/.style={-{Latex[length=2.2mm]}, thick, shorten >=2pt, shorten <=2pt}
]
\node[stage] (sample)
  {Sample $G$ responses\\for the same input $x$\\$\{\tau_i\}_{i=1}^{G}$};
\node[stage, right=of sample] (score)
  {Compute group-relative advantages\\$\sum_i\widehat A_i=0$};
\node[stage, right=of score] (weight)
  {Build effective token modulation\\$W_{i,t}=\alpha_{i,t}w_i$};

\node[op, below=of score] (transform)
  {Apply stop-gradient\\within-group transform\\$\widetilde{\mathbf W}_{t^\star}=\mathcal T(\mathbf W_{t^\star},\widehat{\mathbf A})$};
\node[method, below left=0.45cm and 0.08cm of transform] (minrep)
  {Min-Replace\\$\widetilde W_{i,t^\star}=\min_j W_{j,t^\star}$};
\node[method, below right=0.45cm and 0.08cm of transform] (orth)
  {Orth-Proj\\$\widehat{\mathbf A}^{\top}\widetilde{\mathbf W}_{t^\star}=0$};

\node[good, right=of transform] (grad)
  {Use decoupled gradient estimator\\$\frac1G\sum_i\sum_t
  \widehat A_i\widetilde W_{i,t}\nabla\log\pi_\theta(y_{i,t}\mid h_t^{(i)})$};

\node[warn, below=1.15cm of minrep, xshift=-0.10\textwidth] (before)
  {Before transform\\$\sum_i \widehat A_i W_{i,t^\star}\mathbf g_{\mathrm{shared}}\neq \mathbf 0$\\asymmetric modulation $\Rightarrow$ learning tax};
\node[good, right=0.08\textwidth of before] (after)
  {After transform\\$\sum_i \widehat A_i\widetilde W_{i,t^\star}\mathbf g_{\mathrm{shared}}\approx \mathbf 0$\\restored shared-token cancellation};

\draw[arr] (sample) -- (score);
\draw[arr] (score) -- (weight);
\draw[arr] (weight.south west) -- (transform.north east);
\draw[arr] (transform.east) -- (grad.west);
\draw[arr] (transform.south) -- (minrep.north);
\draw[arr] (transform.south) -- (orth.north);
\draw[arr] (before) -- node[above, font=\scriptsize, fill=white, inner sep=1pt] {structural correction} (after);
\end{tikzpicture}}
\caption{DFPO method diagram. Multiple responses to the same input are used to compute group-relative advantages and sequence-coupled weights, forming the effective token modulation $W_{i,t}$. DFPO then applies a stop-gradient within-group transformation to obtain $\widetilde W_{i,t}$ for the decoupled gradient estimator. Min-Replace uses a shared within-group scale at the selected position, whereas Orth-Proj decorrelates the modulation vector $\mathbf W_{t^\star}$ from the advantage vector $\widehat{\mathbf A}$. The correction preserves token-gradient directions while restoring, or approximating, shared-token cancellation and reducing the learning tax.}
\label{fig:dfpo-method}
\end{figure}


\section{Group size \texorpdfstring{$G$}{G} and relative gains}
\label{app:group-size-ablation}

Table~\ref{tab:ablation_g_wggn_minreplace_qwen3_32b_aime25_horizontal} reports the relative improvement $\Delta$ (DFPO Min-Replace minus GSPO). Increasing $G$ makes trajectories within a group more heterogeneous, making the baseline more susceptible to non-canceling shared/high-frequency tokens and accumulated learning tax. DFPO restores or approximates cancellation in this subspace through within-group alignment, yielding larger relative gains.

\begin{table}[ht]
	\caption{Group-size ablation. Values are DFPO Min-Replace scores minus GSPO scores on Qwen3-32B and AIME25.}
	\label{tab:ablation_g_wggn_minreplace_qwen3_32b_aime25_horizontal}
	\centering
	\small
	\setlength{\tabcolsep}{8pt}
	\renewcommand{\arraystretch}{1.15}
	\begin{tabular}{c c c c c}
		\toprule
		\textbf{G} & 2 & 4 & 8 & 16 \\
		\midrule
		\textbf{avg@32} & 3.1 & 4.2 & 5.6 & 5.7 \\
		\bottomrule
	\end{tabular}
\end{table}

\section{A unified toy problem: learning tax and entropy collapse}
\label{sec:toy-problem}

Using GSPO as an example, we present a unified toy problem to illustrate the two core failure modes identified in Proposition~\ref{prop:violate_cancel_implies_drift}: (i) the learning tax caused by ineffective updates to reward-irrelevant tokens, and (ii) entropy collapse among semantically equivalent correct solutions. The example is intentionally minimal, fully algebraic, and independent of surface forms in any particular language.

\subsection{Prompt and trajectories}

Consider the prompt
\[
x: \text{"What is 11*12?"}
\]
and the following three trajectories:
\begin{align}
	\tau_1 &= (\text{"The answer is 122."}) \qquad \text{(incorrect)}, \\
	\tau_2 &= (\text{"The answer is 132."}) \qquad \text{(correct)}, \\
	\tau_3 &= (\text{"11*12=132."}) \qquad \text{(correct)}.
\end{align}
The reward depends only on semantic correctness of the final answer:
\begin{equation}
	R(\tau_1) = 0, \quad R(\tau_2) = R(\tau_3) = 1.
\end{equation}

We consider a group of size $G=3$ containing the above trajectories. The group-relative advantage is defined as
\begin{equation}
	\widehat{A}_i = R(\tau_i) - \frac{1}{3} \sum_{j=1}^3 R(\tau_j),
\end{equation}
which gives
\begin{equation}
	\widehat{A}_1 = -\frac{2}{3}, \quad \widehat{A}_2 = \widehat{A}_3 = \frac{1}{3}.
	\label{eq:toy-advantages}
\end{equation}

Let $W_{i,t}(\theta)$ denote the effective token modulation of the $t$-th token in trajectory $i$.

\subsection{Learning tax: ineffective updates to reward-irrelevant tokens}
\label{sec:toy-learning-tax}

\paragraph{Minimal case ($G=2, T=3$): sequence coupling prevents shared-prefix gradients from canceling.}
To highlight the structural root, first consider the minimal group containing two trajectories $\{\tau_1,\tau_2\}$, whose group-relative advantages satisfy the zero-mean constraint:
\begin{equation}
	\widehat{A}_2 = -\widehat{A}_1 \triangleq A, \quad A > 0.
\end{equation}
Write the two trajectories as token sequences of length $T=3$:
\begin{align}
	\tau_1 &= (a_1 = \text{"answer"}, a_2 = \text{"is"}, a_3 = \text{"122"}), \\
	\tau_2 &= (a_1 = \text{"answer"}, a_2 = \text{"is"}, a_3 = \text{"132"}).
\end{align}
Assume that the first two steps have the same context-token pairs:
\begin{equation}
	r_{1,1} = r_{2,1} = \rho_1, \quad r_{1,2} = r_{2,2} = \rho_2,
\end{equation}
but the final token differs:
\begin{equation}
	r_{1,3} = \lambda_1, \quad r_{2,3} = \lambda_2, \quad \lambda_1 \neq \lambda_2.
\end{equation}
Under sequence-coupled weighting, the difference at the final step typically induces different effective modulations at the shared prefix, so $W_{1,t}\neq W_{2,t}$. Therefore, for any shared prefix token, $t=1$ or $t=2$, the effective gradient modulation along $\nabla_\theta \log \pi_\theta(a_t \mid h_t)$ is proportional to
\begin{equation}
	\widehat{A}_1 W_{1,t} + \widehat{A}_2 W_{2,t}
	= A\big(W_{2,t}-W_{1,t}\big),
\end{equation}
which is strictly nonzero when $W_{1,t}\neq W_{2,t}$.
This shows that \emph{even when the reward depends only on the final token, sequence-coupled importance weighting can prevent shared-prefix gradients from canceling within the group}, resulting in systematic updates to reward-irrelevant tokens. This is the minimal counterexample for the learning tax.

\subsection{Entropy collapse: probability drift among equivalent correct solutions}
\label{sec:toy-entropy}

Although $\tau_2$ and $\tau_3$ have different surface forms, they express the same mathematical fact and are both correct. They also have the same group-relative advantage: $\widehat{A}_2 = \widehat{A}_3 = \frac{1}{3}$. However, because of differences in tokenization, length, and local likelihoods, their sequence-level weights are generally different:
\begin{equation}
	s_2(\theta) \neq s_3(\theta).
\end{equation}
In the linear approximation region of the update,
\begin{equation}
	\Delta \log \pi_\theta(\tau_i) \approx \eta \cdot c \cdot \widehat{A}_i s_i(\theta), \quad c > 0.
\end{equation}
The log-probability ratio between the two correct trajectories becomes
\begin{equation}
	\log \frac{\pi_{\theta^+}(\tau_2)}{\pi_{\theta^+}(\tau_3)} \approx \log \frac{\pi_{\theta}(\tau_2)}{\pi_{\theta}(\tau_3)} + \eta c \left( \widehat{A}_2 s_2(\theta) - \widehat{A}_3 s_3(\theta) \right).
	\label{eq:toy-ratio-drift}
\end{equation}
Even when $\widehat{A}_2 = \widehat{A}_3$, any small difference between $s_2(\theta)$ and $s_3(\theta)$ causes the probability ratio to drift. Repeated updates accumulate this drift multiplicatively, concentrating probability mass on one surface form and reducing the policy entropy over semantically equivalent correct answers. This is observed as entropy collapse.

\subsection{Discussion}

This unified toy problem shows that the learning tax and entropy collapse share the same structural root: sequence-coupled importance weighting breaks token-level gradient symmetry. The example is deliberately concise, yet it captures the key failure modes observed in large-scale training. These phenomena arise independently of language, reward design, or implementation details, highlighting their structural nature.

\section{Implementation details: a within-group reweighting instance for sequence-coupled objectives}
\label{sec:gvpo}

The goal of this section is not to propose an engineering improvement for a particular baseline, but to \textbf{validate the structural proposition that sequence-coupled weights break token-level gradient symmetry}. We choose a representative group-relative reinforcement-learning algorithm as a vehicle for instantiating this proposition. Specifically, we use GSPO as the analysis and implementation platform because its objective contains both (i) group-relative advantages and (ii) sequence-level, coupled importance weighting, directly supporting the structural phenomena outlined in Section~\ref{sec:proposition}.

Without changing GSPO's \emph{group-relative advantage estimator} or \emph{sequence-level clipping} framework, we apply a \emph{minimal within-group transformation} to the sequence-level importance-ratio vector in each group, obtaining DFPO, or Drift Fixing Policy Optimization. The sole purpose of this transformation is to enforce the key orthogonality condition $\sum_i \widehat{A}_i\widetilde W_{i,t^\star}=0$ within the group, thereby restoring the gradient cancellation structure on shared tokens and turning the structural source of learning tax and entropy collapse into an empirically testable difference.

\subsection{GSPO formulation}

GSPO uses the following sequence-level optimization objective:
\begin{align}
   	\mathcal{J}_\text{GSPO} (\theta) =
   	\mathbb{E}_{ x \sim \mathcal{D},\, \{y_i\}_{i=1}^G \sim \pi_{\theta_\text{old}}( \cdot \mid x) }
   	\left[
   	\frac{1}{G} \sum_{i=1}^{G}
   	\min \left( s_{i}(\theta)\widehat{A}_{i}, \, \mathrm{clip} \left( s_{i}(\theta), 1 - {\varepsilon}, 1 + {\varepsilon} \right)\widehat{A}_{i} \right)
   	\right],
   	\label{eq:gspo}
\end{align}
where the within-group relative advantage estimator is
\begin{align}
   	\widehat{A}_{i} =
   	\frac{r(x, y_i) - \mathrm{mean} \left( \{ r(x, y_i) \}_{i=1}^G \right)}
   	{ \mathrm{std} \left( \{ r(x, y_i) \}_{i=1}^G \right) } ,
   	\label{eq:adv_gspo}
\end{align}
and the sequence-level importance-sampling ratio is defined by normalized sequence likelihood:
\begin{align}
   	s_{i}(\theta)
   	&=
   	\left( \frac{ \pi_{\theta} (y_i \mid x) }{ \pi_{\theta_\text{old}} (y_i \mid x)} \right)^{\frac{1}{|y_i|}}
   	=
   	\exp \left(
   	\frac{1}{|y_i|}
   	\sum_{t=1}^{|y_i|}
   	\log \frac{ \pi_{\theta} (y_{i,t} \mid x, y_{i,<t}) }
   	{ \pi_{\theta_\text{old}} (y_{i,t} \mid x,y_{i,<t})}
   	\right).
   	\label{eq:si_def}
\end{align}
By definition, the group-relative advantages satisfy the zero-mean condition $\sum_{i=1}^G \widehat{A}_i = 0$.

\subsection{DFPO: a within-group transformation that restores shared-token gradient cancellation}

Let the within-group sequence-weight vector, advantage vector, and effective token-gradient modulation vector at a shared token position $t^\star$ be
\begin{align}
   	\mathbf{s}(\theta) \triangleq (s_1(\theta),\dots,s_G(\theta))^\top,\qquad
	\widehat{\mathbf{A}} \triangleq (\widehat{A}_1,\dots,\widehat{A}_G)^\top,\qquad
	\mathbf W_{t^\star}(\theta)
	\triangleq
	\boldsymbol{\alpha}_{t^\star}\odot\mathbf s(\theta).
   	\label{eq:vec_def}
\end{align}
Here $\boldsymbol{\alpha}_{t^\star}=(\alpha_{1,t^\star},\dots,\alpha_{G,t^\star})^\top$. For length-averaged GSPO, $\alpha_{i,t^\star}=1/|y_i|$. If an algorithm uses a different token aggregation rule, then $\alpha_{i,t^\star}$ is the effective aggregation coefficient corresponding to that implementation. The core idea of DFPO is to transform the effective modulation of shared-token gradient terms within each group, $\mathbf W_{t^\star}(\theta)\mapsto \widetilde{\mathbf W}_{t^\star}(\theta)$, and to use $\widetilde W_{i,t^\star}$ as a decoupled modulation at shared positions in the gradient estimator. This restores, or approximately restores, within-group gradient cancellation on shared tokens.

Because clipping is present, we first need to extract $\widehat{A}$, clip the weight, and then apply the within-group transformation. Note that
\[
\min(s_i\widehat{A}_i,\ \mathrm{clip}(s_i)\widehat{A}_i)
\neq
\widehat{A}_i\min(s_i,\ \mathrm{clip}(s_i))
\quad\text{when } \widehat{A}_i<0.
\]
The equality fails because a negative $\widehat{A}_i$ reverses the inequality. We therefore need a \emph{sign-consistent} rewriting.

\subsection{Step 1: rewrite GSPO as "clip first, then multiply by the advantage"}

Define
\begin{equation}
   	c_i(\theta)\triangleq \mathrm{clip}\!\left(s_i(\theta),\,1-\varepsilon,\,1+\varepsilon\right),
   	\qquad
   	s_i(\theta)=\left(\frac{\pi_\theta(y_i\mid x)}{\pi_{\theta_{\text{old}}}(y_i\mid x)}\right)^{\frac{1}{|y_i|}}.
   	\label{eq:def_si_ci}
\end{equation}

We introduce a \textbf{sign-aware post-clipping weight}:
\begin{equation}
   	\bar{s}_i(\theta)
   	\triangleq
   	\begin{cases}
   		\min\!\big(s_i(\theta),\,c_i(\theta)\big), & \widehat{A}_i \ge 0,\\[4pt]
   		\max\!\big(s_i(\theta),\,c_i(\theta)\big), & \widehat{A}_i < 0.
   	\end{cases}
   	\label{eq:postclip_bar_s_piecewise}
\end{equation}

The original GSPO objective
\begin{equation}
   	\mathcal{J}_\text{GSPO}(\theta)
   	=
   	\mathbb{E}\!\left[
   	\frac{1}{G}\sum_{i=1}^G
   	\min\!\big(s_i(\theta)\widehat{A}_i,\ c_i(\theta)\widehat{A}_i\big)
   	\right]
   	\label{eq:gspo_original_again}
\end{equation}
can then be rewritten exactly as
\begin{equation}
   	\mathcal{J}_\text{GSPO}(\theta)
   	=
   	\mathbb{E}\!\left[
   	\frac{1}{G}\sum_{i=1}^G
   	\widehat{A}_i\,\bar{s}_i(\theta)
   	\right].
   	\label{eq:gspo_equiv_factored}
\end{equation}

\paragraph{Equivalence in one line.}
When $\widehat{A}_i\ge 0$, $\min(s_i\widehat{A}_i,c_i\widehat{A}_i)=\widehat{A}_i\min(s_i,c_i)$. When $\widehat{A}_i<0$, $\min(s_i\widehat{A}_i,c_i\widehat{A}_i)=\widehat{A}_i\max(s_i,c_i)$. This yields \eqref{eq:gspo_equiv_factored}.

\subsection{Step 2: apply a within-group transformation to the post-clipping effective modulation vector}

The following two transformations correspond to Min-Replace and Orth-Proj defined earlier. Both act on the post-clipping effective modulation vector at the shared position:
\[
	\bar{\mathbf W}_{t^\star}\triangleq \boldsymbol{\alpha}_{t^\star}\odot\bar{\mathbf{s}}.
\]

\subsubsection{Transformation 1: Min-Replace, or equalizing the post-clipping effective modulation}
\begin{equation}
	\bar W_{\min,t^\star}(\theta)\triangleq \min_{j\in\{1,\dots,G\}}\bar W_{j,t^\star}(\theta),
	\qquad
	\widetilde W_{i,t^\star}(\theta)\triangleq \bar W_{\min,t^\star}(\theta)\ \ \forall i.
   	\label{eq:gvpo_min_on_bar}
\end{equation}

\subsubsection{Transformation 2: Positive Orth-Proj, or Orth-Proj with a nonnegativity constraint}
\label{sec:gvpo_positive_orthproj}

In practice, effective modulation coefficients must be nonnegative. The unconstrained orthogonal projection in \eqref{eq:method_orthproj} may produce negative components. We therefore replace unconstrained Orth-Proj with a \textbf{minimal-perturbation projection in the nonnegative domain}: keep $\widetilde{\mathbf W}_{t^\star}\succeq \mathbf{0}$, make it as close as possible to the original post-clipping effective modulation $\bar{\mathbf W}_{t^\star}$, and simultaneously satisfy
\[
	\widehat{\mathbf A}^{\top}\widetilde{\mathbf W}_{t^\star}=0.
\]
For length-averaged GSPO, $\bar{\mathbf W}_{t^\star}=\boldsymbol{\alpha}_{t^\star}\odot\bar{\mathbf{s}}$ already includes the length-aware scaling $\alpha_{i,t^\star}=1/|y_i|$.

Specifically, \textbf{Positive Orth-Proj} is defined as the following quadratic program:
\begin{equation}
	\widetilde{\mathbf W}_{t^\star}(\theta)
	=
	\arg\min_{\mathbf{v}\in\mathbb{R}^G}
	\ \frac{1}{2}\|\mathbf{v}-\bar{\mathbf W}_{t^\star}(\theta)\|_2^2
	\quad
	\text{s.t.}\quad
	\widehat{\mathbf A}^{\top} \mathbf{v}=0,\ \ \mathbf{v}\succeq \mathbf{0}.
	\label{eq:gvpo_posproj_qp}
\end{equation}

\paragraph{Feasibility and implementation cost.}
Because group-relative advantages satisfy $\sum_{i=1}^G \widehat{A}_i=0$, any constant vector $c\mathbf{1}$ satisfies $\widehat{\mathbf A}^{\top}(c\mathbf{1})=0$. Therefore, \eqref{eq:gvpo_posproj_qp} always has a nonnegative feasible solution. When the advantage vector contains both positive and negative components, the projection can be interpreted as redistributing effective modulation coefficients between the positive- and negative-advantage subsets to achieve a zero inner product. If the degenerate case $\widehat{\mathbf A}=\mathbf 0$ occurs, the constraint is inactive and we keep $\bar{\mathbf W}_{t^\star}$ unchanged.

In implementation, we solve \eqref{eq:gvpo_posproj_qp} directly. The QP dimension is only the group size $G$. Since training already requires computing advantages, weights, and within-group aggregation for each group, the additional cost of Positive Orth-Proj is negligible relative to model forward and backward passes.

\paragraph{Discussion.}
Equation~\eqref{eq:gvpo_posproj_qp} exactly satisfies the group orthogonality constraint $\widehat{\mathbf A}^{\top}\widetilde{\mathbf W}_{t^\star}=0$, equivalently $\sum_i\widehat{A}_i\widetilde W_{i,t^\star}=0$, and gives the minimum-$L_2$ perturbation solution relative to $\bar{\mathbf W}_{t^\star}$ in the nonnegative domain. If an implementation needs to map back to sequence-level weights and $\alpha_{i,t^\star}>0$, one can set $\widetilde{s}_i=\widetilde W_{i,t^\star}/\alpha_{i,t^\star}$. This scalar constraint only removes systematic drift in the shared-token subspace and does not make the overall policy gradient zero; see Appendix~\ref{app:nonzero-grad}.


\section{Why \texorpdfstring{$\sum_{i=1}^G \widehat A_i\widetilde W_{i,t^\star}=0$}{sum A W equals zero} does not mean the gradient is zero}
\label{app:nonzero-grad}

This appendix clarifies a common and important point of confusion. Even when we construct effective token-gradient modulation coefficients $\widetilde W_{i,t^\star}$ within each group so that
\begin{equation}
\sum_{i=1}^G \widehat{A}_i\widetilde W_{i,t^\star} = 0,
   	\label{eq:orth_cond}
\end{equation}
this \emph{does not imply} that the corresponding policy-gradient update is zero. Intuitively, \eqref{eq:orth_cond} constrains only the sum of scalar coefficients at the shared token position $t^\star$, while the policy gradient remains a weighted sum of "scalar modulation times each token-gradient vector in each trajectory." Unless these gradient vectors are identical within the group, a zero scalar sum does not necessarily make the vector sum zero. For length-averaged GSPO, $\widetilde W_{i,t}=\alpha_{i,t}\widetilde{s}_i=\widetilde{s}_i/|y_i|$, so length normalization is already included in the effective token-gradient modulation.

\vspace{0.25em}
\paragraph{Key implementation convention: no backpropagation through $\widetilde W_{i,t}$, or stop-gradient.}
In the actual algorithm, $\widetilde W_{i,t}$ is obtained by applying a within-group transformation to the effective token-gradient modulation vector $\mathbf W_t(\theta)$ of the current group samples $(\{y_i\},\{\widehat{A}_i\})$. To ensure that this transformation acts as a structural correction or control variate, we use the \emph{stop-gradient} convention: during backpropagation, $\widetilde W_{i,t}$ is treated as a constant, and the transformation operator itself is not differentiated. This convention is consistent with common practice in PPO/GRPO/GSPO, where samples are drawn from $\pi_{\theta_{\mathrm{old}}}$ and gradients are stopped through the advantage $\widehat{A}_i$.

Under this convention, \eqref{eq:orth_cond} does not make the overall gradient degenerate. Instead, it restores the cancellation structure in the subspace of shared token gradient directions, as shown below.

\subsection{Review of the GSPO gradient without clipping terms}

For comparison, we start from GSPO. In a linear segment, the GSPO objective is
\begin{equation}
   	\mathcal{J}_{\mathrm{GSPO}}(\theta)
   	=
   	\mathbb{E}_{x\sim\mathcal{D},\,\{y_i\}_{i=1}^G\sim \pi_{\theta_{\mathrm{old}}}(\cdot|x)}
   	\left[
   	\frac{1}{G}\sum_{i=1}^G s_i(\theta)\widehat{A}_i
   	\right],
\end{equation}
where
\begin{equation}
   	s_i(\theta)
   	=
   	\left(\frac{\pi_\theta(y_i|x)}{\pi_{\theta_{\mathrm{old}}}(y_i|x)}\right)^{\frac{1}{|y_i|}}
   	=
   	\exp\!\left(
   	\frac{1}{|y_i|}
   	\sum_{t=1}^{|y_i|}
   	\log\frac{\pi_\theta(y_{i,t}|x,y_{i,<t})}{\pi_{\theta_{\mathrm{old}}}(y_{i,t}|x,y_{i,<t})}
   	\right).
   	\label{eq:s_def_again}
\end{equation}
Using the log-derivative trick, the gradient is
\begin{align}
   	\nabla_{\theta} \mathcal{J}_\text{GSPO} (\theta)
   	=
   	\mathbb{E}\!\left[
   	\frac{1}{G} \sum_{i=1}^{G}
   	s_{i}(\theta) \widehat{A}_{i}
   	\cdot \frac{1}{|y_i|} \sum_{t=1}^{|y_i|} \nabla_{\theta} \log \pi_{\theta} (y_{i,t} \mid x, y_{i,<t})
   	\right].
   	\label{eq:gspo_grad_again}
\end{align}

By contrast, the gradient of GRPO in token-decomposed form is
\begin{align}
   	\nabla_{\theta} \mathcal{J}_\text{GRPO} (\theta)
   	=
   	\mathbb{E}\!\left[
   	\frac{1}{G} \sum_{i=1}^{G} \widehat{A}_{i}
   	\cdot \frac{1}{|y_i|} \sum_{t=1}^{|y_i|}
   	\frac{ \pi_{\theta} (y_{i,t} \mid x, y_{i,<t}) }{ \pi_{\theta_\text{old}} (y_{i,t} \mid x,y_{i,<t})}
   	\nabla_{\theta} \log \pi_{\theta} (y_{i,t} \mid x, y_{i,<t})
   	\right].
   	\label{eq:grpo_grad_again}
\end{align}

\subsection{Gradient after the within-group transformation: why it is not zero}

Now consider applying a within-group transformation $\mathbf W_t \mapsto \widetilde{\mathbf W}_t$ to the effective token-gradient modulation and applying \emph{stop-gradient} to $\widetilde{\mathbf W}_t$ in the update. During gradient computation, $\widetilde W_{i,t}$ is treated as a constant coefficient and is not differentiated.

The GSPO gradient in the DFPO-like form is
\begin{align}
	\nabla_{\theta} \widetilde{\mathcal{J}} (\theta)
	=
	\mathbb{E}\!\left[
	\frac{1}{G}\sum_{i=1}^{G}\sum_{t=1}^{|y_i|}
	\widehat{A}_{i}\,\widetilde W_{i,t}
	\nabla_{\theta} \log \pi_{\theta} (y_{i,t} \mid x, y_{i,<t})
	\right],
	\label{eq:gvpo_like_grad}
\end{align}
where $\widetilde W_{i,t}$ is not backpropagated with respect to $\theta$ and is treated as a constant coefficient. For length-averaged GSPO, if the underlying implementation still uses sequence weights, then $\widetilde W_{i,t}=\alpha_{i,t}\widetilde{s}_i=\widetilde{s}_i/|y_i|$.

Notice that \eqref{eq:gvpo_like_grad} is a weighted sum of \emph{vectors}:
\[
\sum_i\sum_t \widehat A_i\widetilde W_{i,t}\nabla_\theta \log\pi_\theta(\cdot).
\]
Even if the scalar constraint \eqref{eq:orth_cond} holds at a shared position, it does not imply that the vector sum is zero unless the gradient vectors inside the sum are identical within the group. This is the key to understanding why \eqref{eq:orth_cond} restores cancellation for shared tokens but does not remove the overall learning signal.

\subsection{Formal decomposition: which terms cancel and which do not}

Let the gradient direction of a single token be
\begin{equation}
   	\mathbf{g}_{i,t}(\theta)
   	\triangleq
   	\nabla_\theta \log \pi_{\theta}(y_{i,t}\mid x,y_{i,<t}).
\end{equation}
Expanding \eqref{eq:gvpo_like_grad} gives
\begin{equation}
   	\nabla_{\theta} \widetilde{\mathcal{J}} (\theta)
   	=
   	\mathbb{E}\!\left[
   	\frac{1}{G}\sum_{i=1}^G\sum_{t=1}^{|y_i|}
\underbrace{\left(\widehat{A}_i\widetilde W_{i,t}\right)}_{\text{scalar coefficient}}
   	\mathbf{g}_{i,t}(\theta)
   	\right].
   	\label{eq:grad_vec_sum}
\end{equation}

\paragraph{Exact cancellation on the shared-token subset.}
Consider a shared context-token event: there exists a fixed context-action pair $(h^\star,a^\star)$ such that all trajectories in the group have a matching token at the corresponding time step:
\[
(y_{i,<t},y_{i,t})=(h^\star,a^\star).
\]
They therefore produce the same gradient direction:
\begin{equation}
   	\mathbf{g}_{i,t}(\theta)=\mathbf{g}^\star(\theta)
	\quad \text{for all }(i,t)\text{ in the shared event}.
   	\label{eq:shared_grad_dir}
\end{equation}
Extracting these shared terms from \eqref{eq:grad_vec_sum} gives their combined contribution:
\begin{equation}
   	\mathbf{g}^\star(\theta)\cdot
   	\left(
   	\sum_{(i,t)\in \mathcal{S}(h^\star,a^\star)}
\widehat{A}_i\widetilde W_{i,t}
   	\right),
   	\label{eq:shared_contrib}
\end{equation}
where $\mathcal{S}(h^\star,a^\star)$ denotes the index set of the shared context-token pair.

Therefore, when the shared event corresponds to the same token position $t^\star$ within the group and we construct $\widetilde{\mathbf W}_{t^\star}$ such that
\begin{equation}
\sum_{i=1}^G \widehat A_i\widetilde W_{i,t^\star}=0,
\end{equation}
the combined gradient at that shared position cancels exactly. For length-averaged GSPO, $\widetilde W_{i,t^\star}=\widetilde{s}_i/|y_i|$ explicitly carries length normalization, so the constraint directly corresponds to the true token-gradient coefficient. This forms the theoretical basis for how our within-group transformation restores token-level symmetry, or exchangeability.

\paragraph{Non-shared token gradients are generally nonzero.}
For generic tokens, however, $\mathbf{g}_{i,t}(\theta)$ differs across trajectories and time steps. Thus, \eqref{eq:grad_vec_sum} is a weighted sum of vectors pointing in different directions. Even if the scalar constraint \eqref{eq:orth_cond} holds, these different vectors need not cancel completely. Formally,
\begin{equation}
\sum_{i=1}^G \widehat A_i\widetilde W_{i,t^\star}=0
   	\ \not\Rightarrow\
\sum_{i,t}\widehat A_i\widetilde W_{i,t}\mathbf{g}_{i,t}(\theta)=\mathbf{0},
\end{equation}
unless all $\mathbf{g}_{i,t}(\theta)$ are collinear and the coefficients are exactly matched, which is a highly degenerate case.

Therefore, the role of \eqref{eq:orth_cond} is not to make the update zero. Rather, it \emph{selectively} removes or suppresses systematic drift from tokens that repeatedly appear within the group, are weakly correlated with the reward, and have highly aligned gradient directions. This reduces the learning tax and slows entropy collapse while preserving learning signals for genuinely reward-relevant decision tokens.

\section{Learning tax: accumulated ineffective updates and their consequences}
\label{app:learning-tax-consequences}

\paragraph{Background and definition.}
In sequence-based reinforcement-learning fine-tuning with terminal rewards, including group-relative methods and their variants, we define \emph{ineffective updates}, or the \emph{learning tax}, as parameter updates that continue to occur but statistically \textbf{do not produce a net gain in the target capability}; they may even be unrelated to the target. This mainly appears as a lower signal-to-noise ratio (SNR) of the gradient, a mismatch between the update direction and true credit assignment, and large updates that are absorbed by clipping or constraint mechanisms without making useful progress. When such ineffective updates accumulate during training, they have systematic consequences for optimization dynamics, capability generalization, memory stability, and engineering cost. We discuss these consequences below.

\subsection{Optimization dynamics: degraded convergence and training instability}
\paragraph{(1) Effective signal is drowned by noise: higher gradient variance and lower sample efficiency.}
Ineffective updates inject task-irrelevant noise into the parameters, reducing the proportion of effective gradient under the same compute budget. This directly causes slower convergence, a larger number of tokens and optimization steps to reach the same performance threshold, noisier training curves, and longer plateaus.

\paragraph{(2) Adaptive optimizer statistics are contaminated: momentum and second-moment estimates drift.}
For adaptive optimizers such as Adam, AdamW, and Adafactor, parameter updates depend on historical first-moment and second-moment estimates. Long-term accumulation of ineffective gradients causes momentum directions and scaling factors to be dominated by \emph{incorrect gradient statistics}, which then mis-scales the true gradient. In practice, this manifests as the need for a smaller learning rate to remain stable, higher sensitivity to random seeds, and more frequent training oscillations.

\paragraph{(3) Clipping and trust-region mechanisms are over-triggered: effective updates are weakened.}
In PPO/GRPO/GSPO-style objectives, abnormal importance ratios and policy drift trigger clipping or KL/trust-region constraints. When ineffective updates make the importance-ratio distribution more extreme, clipping activates more frequently and \emph{also weakens effective signals}, creating a negative feedback loop: more updates lead to more clipping, and more clipping leaves less signal.

\subsection{Capabilities and generalization: mode collapse, reward hacking, and degraded generalization}
\paragraph{(1) Entropy collapse and mode collapse: reduced exploration and diversity.}
When training repeatedly reinforces surface patterns that correlate with reward but are not causally tied to the true reasoning process, such as fixed templates, surface formatting, or redundant phrasing, the model gradually contracts to a small number of high-probability modes. This reduces output diversity and severely harms exploration and branch search in long-chain reasoning.

\paragraph{(2) Reward hacking becomes easier.}
Ineffective updates push the model toward paths that obtain high reward more easily without genuinely solving the task, such as catering to evaluator preferences, over-explaining, or formatting outputs in a particular way. This often improves offline metrics while degrading true task quality, especially in out-of-distribution evaluations.

\paragraph{(3) Overfitting to spurious features blocks the learning of transferable reasoning operators.}
When credit assignment is incorrect, the model is more likely to learn dataset biases, prompt triggers, or surface correlations rather than transferable reasoning operators. This can yield stable performance on the training set or in-distribution validation set while substantially reducing cross-task transfer and robustness.

\subsection{Memory stability: catastrophic forgetting and capability drift}
\paragraph{(1) Catastrophic forgetting: existing capabilities are damaged by ineffective perturbations.}
Ineffective updates impose persistent perturbations on many parameters and damage the structural integrity of existing capabilities, especially capabilities associated with fragile equilibria such as linguistic fluency, factual consistency, and aligned behavior. Empirically, this often appears as improvement on some benchmarks together with unexplained degradation of seemingly unrelated capabilities.

\paragraph{(2) Capability drift and irreproducibility: high sensitivity to seeds and batches.}
When training is dominated by ineffective updates, the optimization trajectory resembles a momentum-driven random walk. The same setting can then produce substantially different outcomes under different random seeds, unstable model versions, frequent regression-test failures, and higher engineering maintenance cost.

\subsection{Resources and engineering: higher cost and harder subsequent correction}
\paragraph{(1) Diminishing marginal returns: less effective progress under the same compute budget.}
Accumulated learning tax is equivalent to spending expensive online sampling tokens and optimization steps on updates with no net gain, substantially lowering sample efficiency and increasing training cost.

\paragraph{(2) Subsequent alignment or safety correction becomes harder and requires stronger pullback.}
Once ineffective updates push the model away from the original parameter basin, later corrections with SFT or preference-alignment data require stronger training intensity, which can introduce new side effects such as overfitting to alignment data or further forgetting of base capabilities.

\subsection{Behavioral level: inconsistent long-horizon reasoning and weaker self-correction}
\paragraph{(1) Reasoning-chain consistency is lost: intermediate steps contradict each other more often.}
Ineffective updates damage token-level or step-level consistency constraints, causing more logical jumps, contradictions, and unnecessary reasoning branches in long-chain reasoning.

\paragraph{(2) Stable self-correction loops are harder to form: the same errors recur.}
If updates cannot be accurately attributed to erroneous tokens or steps, the model struggles to establish a stable error-detection and correction mechanism. The same errors then recur after training, making the reasoning process less trustworthy and difficult to repair with a small amount of additional training.

\paragraph{Summary.}
The learning tax is not merely a waste of compute. Its long-term accumulation contaminates optimizer statistics, triggers clipping feedback, induces entropy collapse and reward hacking, exacerbates catastrophic forgetting, and ultimately degrades model capability, generalization, and version stability. This phenomenon provides a unified mechanism for understanding training instability and low sample efficiency in long-horizon reasoning, and motivates the design of process-level credit assignment or within-group consistent weighting strategies.

\section{Why sign-asymmetric clipping breaks within-group cancellation}
\label{app:grpo-clip}

\begin{corollary}[GRPO-style clipping can cause within-group cancellation failure outside the clipping interval]
	\label{cor:clip-break-cancel}
	Consider the clipped surrogate objective commonly used in GRPO, or its equivalent rewriting, where piecewise selection is determined by the sign of the advantage.
	Let $\bar w = \mathrm{clip}(w, 1-\varepsilon, 1+\varepsilon)$. For a scalar weight $w$ and advantage $A$, the sign-sensitive equivalent form is
	\[
	\min(wA, \bar w A) =
	\begin{cases}
		A \cdot \min(w, \bar w), & \text{if } A \ge 0,\\
		A \cdot \max(w, \bar w), & \text{if } A < 0.
	\end{cases}
	\]
	In a shared context-token event, suppose there are two trajectories $i,j$ in the group such that
	$w_{i,t^\star} = w_{j,t^\star} = w$, $\widehat{A}_i > 0$, and $\widehat{A}_j < 0$.
	If $w \notin [1-\varepsilon, 1+\varepsilon]$ causes the two trajectories to enter \emph{different branches} of the piecewise operator, for example one in the unclipped branch and the other in the clipped-constant branch, then the true gradient coefficients of the shared token are no longer identical within the group. This can cause within-group cancellation to fail and produce nonzero drift.
\end{corollary}

For example, let $G = 2$ and $\widehat{A}_1 = -A, \widehat{A}_2 = +A$. At the shared-token position, set $\nabla_\theta w = w\,\mathbf g^\star(\theta)$, where $\mathbf g^\star(\theta)=\nabla_\theta\log\pi_\theta(a^\star\mid h^\star)$.

\begin{itemize}
	\item If $w > 1 + \varepsilon$, then $\bar w = 1 + \varepsilon$. For the positive-advantage sample, $\min(wA, \bar w A) = \bar w A$; for the negative-advantage sample, $\min(w(-A), \bar w(-A)) = w(-A)$. Thus, the positive-advantage sample is in the clipped-constant branch and has true gradient $0$, while the negative-advantage sample is in the unclipped branch and has true gradient $-A w\,\mathbf g^\star(\theta)$. The true group-aggregated gradient of the shared token in this direction is
	\begin{equation}
		\big(-A w + 0\big)\,\mathbf g^\star(\theta)
		= -A w\,\mathbf g^\star(\theta)
		\neq \mathbf 0
		\quad (A>0,\ w>0,\ \mathbf g^\star(\theta)\neq \mathbf 0).
		\label{eq:clip-break-cancel-upper}
	\end{equation}

	\item If $w < 1 - \varepsilon$, then $\bar w = 1 - \varepsilon$. For the positive-advantage sample, $\min(wA, \bar w A) = wA$; for the negative-advantage sample, $\min(w(-A), \bar w(-A)) = \bar w(-A)$. Thus, the positive-advantage sample is in the unclipped branch and has true gradient $A w\,\mathbf g^\star(\theta)$, while the negative-advantage sample is in the clipped-constant branch and has true gradient $0$. The true group-aggregated gradient of the shared token in this direction is
	\begin{equation}
		\big(0 + A w\big)\,\mathbf g^\star(\theta)
		= A w\,\mathbf g^\star(\theta)
		\neq \mathbf 0
		\quad (A>0,\ w>0,\ \mathbf g^\star(\theta)\neq \mathbf 0).
		\label{eq:clip-break-cancel-lower}
	\end{equation}
\end{itemize}

In summary, \emph{clipping preserves within-group cancellation of shared tokens only when $w \in [1-\varepsilon,1+\varepsilon]$. Once $w$ leaves the clipping interval, the sign asymmetry of clipping makes the true gradient coefficients of within-group samples no longer interchangeable between the unclipped and clipped-constant branches, producing nonzero drift and learning tax.}

\section{Nonzero drift in the statistical significance of learning tax: from exact cancellation to an expectation-based lower bound}
\label{app:tax-lowerbound}

\begin{corollary}[Non-exchangeable weights induce nonzero expected drift on generic tokens]
	\label{cor:tax-lowerbound}
	Consider a class of frequent tokens $\mathcal{C}$, such as template words or function words, that are \emph{weakly correlated} with the terminal reward. Consider the goal of statistical cancellation during training:
	\[
	\mathbb{E}\!\left[\widehat{A} \mid y \in \mathcal{C}\right] \approx 0.
	\]
	If the gradient modulation of within-group tokens breaks exchangeability, so that under the condition $y_{i,t} \in \mathcal{C}$, the effective coefficient $W_{i,t}$ multiplying the token gradient is systematically correlated with the group-comparison signal $\widehat{A}_i$:
	\[
	\mathrm{Cov}\!\left(\widehat{A}_i, W_{i,t}(\theta) \ \middle|\ y_{i,t} \in \mathcal{C}\right) \neq 0,
	\]
	and if token-gradient directions within this token bucket are approximately aligned and non-degenerate in the shared/high-frequency subspace, then the aggregated group gradient of this token class is strictly nonzero in expectation:
	\[
	\mathbb{E}\!\left[
	\frac{1}{G}\sum_{i=1}^G
	\widehat{A}_i\,W_{i,t}(\theta)\,\mathbf g_{i,t}(\theta)
	\ \middle|\ y_{i,t} \in \mathcal{C}
	\right] \neq \mathbf{0}.
	\]
	This creates accumulated reward-irrelevant drift, or learning tax, during training, consistent with the KL-drift conclusion in Proposition~\ref{prop:violate_cancel_implies_drift}.
\end{corollary}

\paragraph{One-sentence summary.}
This corollary captures the typical practical setting where contexts rarely match exactly. Exact cancellation is rare, but if shared-token gradient terms are exchangeable within the group, so that the effective modulation $W_{i,t}$ is not systematically correlated with $\widehat{A}$, statistical averaging can approximately cancel them and keep the learning tax small. Once exchangeability fails and $W_{i,t}$ correlates with $\widehat{A}$, generic tokens exhibit persistent drift.

\section{Symmetric clipping: fixing sign asymmetry in GRPO clipping and restoring within-group exchangeability}
\label{app:sym-grpo-clip}

This appendix provides a \textbf{symmetric clipping} scheme that fixes the \textbf{sign asymmetry} of the GRPO-style clipped surrogate objective discussed in Appendix~\ref{app:grpo-clip}. Other GRPO variants such as DAPO, DCPO, and SSPO can be adjusted similarly. The standard GRPO objective triggers different $\min/\max$ branches for $A>0$ and $A<0$. Consequently, even when all trajectories in a group have the same token ratio $r$, their \emph{effective coefficients} can differ, breaking within-group cancellation of the shared token.

\vspace{0.25em}
\paragraph{Review: asymmetry comes from coupling "$\min$" with the sign of $A$.}
For scalar ratio $r$, advantage $A$, and clipped ratio $\bar r = \mathrm{clip}(r, 1-\varepsilon, 1+\varepsilon)$, the standard surrogate objective is
\begin{equation}
	\mathcal{L}_{\mathrm{ppo}}(r,A)
	\triangleq \min(rA,\ \bar r A).
	\label{eq:ppo_surrogate_scalar}
\end{equation}
Its equivalent piecewise form, derived in Appendix~\ref{app:grpo-clip}, is
\begin{equation}
	\min(rA, \bar r A)
	=
	\begin{cases}
		A \cdot \min(r, \bar r), & A \ge 0, \\
		A \cdot \max(r, \bar r), & A < 0,
	\end{cases}
	\label{eq:ppo_sign_aware_equiv}
\end{equation}
which explicitly depends on the sign of $A$. Therefore, when positive and negative advantages coexist within a group, identical $r$ values can enter different branches for different trajectories, breaking exchangeability and cancellation.

\vspace{0.5em}
\subsection{Symmetric trust-region clipping: decoupling the effective weight from the sign of the advantage}

To restore exchangeability of shared tokens, we rewrite the clipped surrogate as a \textbf{sign-independent} hard trust-region form:
\begin{equation}
	\mathcal{L}_{\mathrm{sym}}(r,A)
	\triangleq A \cdot \phi(r),
	\qquad
	\phi(r) \triangleq \mathrm{clip}(r, 1-\varepsilon, 1 + \varepsilon).
	\label{eq:sym_clip_scalar}
\end{equation}
That is, \textbf{the same two-sided clipping operator $\phi(r)$ is applied to all samples, regardless of the sign of $A$}. In a token-decomposed GRPO/PPOLike objective, we replace each token-level quantity $r_{i,t}$ by $\phi(r_{i,t})$.

\paragraph{A symmetric-clipping version of GRPO-fix (token-decomposed).}
Let the within-group advantages satisfy the zero-mean constraint $\sum_{i=1}^G \widehat{A}_i = 0$, and define the token-level ratio
$r_{i,t}(\theta) = \frac{\pi_\theta(y_{i,t} \mid h^{(i)}_t)}{\pi_{\theta_{\mathrm{old}}}(y_{i,t} \mid h^{(i)}_t)}$.
The symmetric-clipping GRPO objective is
\begin{equation}
	\mathcal{J}_{\mathrm{GRPO\text{-}SymClip}}(\theta)
	=
	\mathbb{E}\Bigg[
	\frac{1}{G}\sum_{i=1}^{G}\sum_{t=1}^{T_i}
	\phi\!\big(r_{i,t}(\theta)\big)\ \widehat{A}_i
	\Bigg],
	\qquad
	\phi(r) = \mathrm{clip}(r, 1-\varepsilon, 1 + \varepsilon).
	\label{eq:grpo_symclip_obj}
\end{equation}
Ignoring subgradient details at clipping breakpoints, its gradient is
\begin{equation}
	\nabla_\theta \mathcal{J}_{\mathrm{GRPO\text{-}SymClip}}(\theta)
	=
	\mathbb{E}\Bigg[
	\frac{1}{G}\sum_{i=1}^{G}\sum_{t=1}^{T_i}
	\widehat{A}_i\ \phi'\!\big(r_{i,t}(\theta)\big)\ r_{i,t}(\theta)
	\nabla_\theta \log \pi_\theta\!\big(y_{i,t} \mid h^{(i)}_t\big)
	\Bigg],
	\label{eq:grpo_symclip_grad}
\end{equation}
where
\[
\phi'(r)=
\begin{cases}
0, & r<1-\varepsilon,\\
1, & 1-\varepsilon<r<1+\varepsilon,\\
0, & r>1+\varepsilon.
\end{cases}
\]
At breakpoints, one may use a subgradient or ignore the measure-zero set. Thus, when the ratio lies outside $[1-\varepsilon,1+\varepsilon]$, the true gradient of this clipped surrogate is zero.

\vspace{0.5em}
\subsection{How symmetric clipping repairs shared-token cancellation}

\begin{corollary}[Symmetric clipping restores within-group cancellation of shared tokens in a minimal structural setting]
	\label{cor:symclip_restore_cancel}
	Fix an input $x$ and a time step $t^\star$. Consider the event $\mathcal{E}_{t^\star}$ in which all trajectories in the group share the same context-token pair $(h^\star,a^\star)$ at this step. Thus, for all $i$, $r_{i,t^\star}(\theta)=r^\star(\theta)$, and the corresponding score function $\nabla_\theta \log \pi_\theta(a^\star \mid h^\star)$ is identical within the group. If the within-group advantages satisfy $\sum_{i=1}^G \widehat{A}_i = 0$, then under the symmetric clipping objective \eqref{eq:grpo_symclip_obj}, the aggregated gradient of this shared token cancels exactly:
	\begin{equation}
		\sum_{i=1}^G \widehat{A}_i\ \phi'\!\big(r_{i,t^\star}(\theta)\big)\ r_{i,t^\star}(\theta)
		\nabla_\theta \log \pi_\theta(a^\star \mid h^\star)
		=
		\phi'\!\big(r^\star(\theta)\big)r^\star(\theta)
		\Big(\sum_{i=1}^G \widehat{A}_i\Big)
		\nabla_\theta \log \pi_\theta(a^\star \mid h^\star)
		= \mathbf{0}.
		\label{eq:symclip_cancel}
	\end{equation}
\end{corollary}

\paragraph{Explanation.}
The key is that $\phi(\cdot)$ and its derivative no longer depend on the sign of $\widehat{A}_i$. Thus, the \emph{gradient modulation} of a shared token remains exchangeable within the group, allowing exact cancellation induced by zero-mean advantages. This precisely fixes the source of exchangeability violation in the standard clipped surrogate objective \eqref{eq:ppo_sign_aware_equiv}.

\vspace{0.5em}
\subsection{Relationship to the standard GRPO surrogate and its cost}

\paragraph{Difference.}
Symmetric clipping in \eqref{eq:sym_clip_scalar} is a sign-independent hard trust-region comparison baseline, not an exact replica of the sign-sensitive one-sided correction mechanism in standard PPO/GRPO. Its design priority is as follows: when $r$ lies outside $[1-\varepsilon,1+\varepsilon]$, samples from the old policy are considered overly off-policy relative to the current policy, and the reliability of the importance weight is reduced. Therefore, the gradient contribution of that sample is set to zero within the current batch. This freezing is not permanent. In PPO/GRPO-style training, the old policy is periodically synchronized as $\theta_{\mathrm{old}}\leftarrow\theta$; after synchronization and resampling, the ratio returns to a region close to $1$, and the sample re-enters the update interval. The main cost of SymClip is therefore more conservative updates, lower sample utilization within a single batch, and potentially slower convergence. Its main benefits are preserving the complete trust region, avoiding domination by invalid IS weights, and maintaining sign-independent shared-token cancellation.

\paragraph{Engineering implementation: recommended minimal change.}
If the current implementation uses a token-decomposed GRPO/PPOLike form, each token's effective weight only needs to be changed from
\[
\min(r_{i,t} \widehat{A}_i, \ \mathrm{clip}(r_{i,t}) \widehat{A}_i)
\]
to
\[
\widehat{A}_i \cdot \mathrm{clip}(r_{i,t}, 1-\varepsilon, 1 + \varepsilon),
\]
which implements the symmetric clipping objective in \eqref{eq:grpo_symclip_obj}.

\section{Implementation details}
\label{app:impl-details}

\paragraph{Models and context lengths.}
We configure \textbf{Qwen3-32B} with a \textbf{32k}-token context length and \textbf{Qwen3-Next-80B-A3B-Thinking} with a \textbf{256k}-token context length. Inference uses the \textbf{vLLM} engine, version \textbf{0.11.2}.

\paragraph{Hardware.}
Experiments are conducted on \textbf{32} \textbf{NVIDIA A800 (80GB)} GPUs.

\paragraph{Optimization hyperparameters.}
Training hyperparameters are as follows:
\begin{itemize}
	\item Initial learning rate: $8 \times 10^{-7}$;
	\item Learning-rate schedule: cosine decay, with the minimum learning-rate ratio set to $0.2$;
	\item Warmup: linear warmup covering $3\%$ of total training steps;
	\item Entropy regularization coefficient: $\beta=0$;
	\item Number of rollouts: 32 trajectories sampled per input;
	\item Mini-batch size: 32.
\end{itemize}

\section{Additional experimental settings}
\label{app:inference-settings}

\paragraph{Inference settings.}
For Qwen3-32B, the decoding parameters are Temperature=$0.6$, TopP=$0.95$, TopK=$20$, and MinP=$0$. For Qwen3-Next-80B-A3B-Thinking, the decoding parameters are Temperature=$0.6$, TopP=$0.95$, TopK=$20$, and MinP=$0$. GPT-OSS-120B follows OpenAI's official recommendation with Temperature=$1.0$ and TopP=$1.0$. GPT-OSS-120B evaluation uses the harmony response format, and the system prompt sets reasoning effort to high (``Reasoning: high''). All methods for the same model are compared under identical decoding settings.

\section{Bias of Min-Replace and failure of importance-sampling unbiasedness: effects, bounds, and testable predictions}
\label{sec:minreplace_no_reverse}

This appendix gives a complete implementation-level and theoretical answer to a key question:

\begin{quote}
	\textbf{Does applying Min-Replace within a group, namely taking the minimum and broadcasting it, break the unbiasedness of importance sampling (IS)? What are its effects? Does it cause direction errors or reverse updates? When can its bias be ignored?}
\end{quote}

The conclusions are as follows, under the default implementation assumption in this paper: \textbf{stop-gradient is applied to within-group transformation coefficients}.

\begin{enumerate}
	\item \textbf{IS unbiasedness is indeed broken:} The estimator produced by Min-Replace is no longer an unbiased gradient of the original sequence-coupled IS objective. It optimizes a more \emph{conservative} surrogate objective.
	\item \textbf{There is no reverse update, or sign flip:} Under stop-gradient, Min-Replace only applies a \emph{proportional shrinkage} to each trajectory's effective modulation coefficient at shared positions. It does not decrease a probability that should be increased; the update sign determined by $\widehat{A}_i$ is preserved.
	\item \textbf{The main effect is a bias-variance trade-off:} Min-Replace strongly suppresses the dominance of large-ratio or tail samples within the group, substantially reducing variance, reducing excessive clipping/KL-constraint triggers, and improving stability. Its cost is bias and a smaller effective step size, which may slow convergence and, in extreme cases, shift the optimum toward the old policy.
	\item \textbf{When the bias is negligible:} If training is constrained to a small trust region such that the effective modulation coefficients $\bar W_{i,t^\star}(\theta)$ at shared positions have low relative dispersion within the group, then the bias upper bound vanishes with that dispersion. Min-Replace can then be viewed approximately as robust variance reduction or an implicit trust region.
\end{enumerate}

\vspace{0.5em}
\subsection{Setup: sequence-coupled effective modulation and Min-Replace, with DFPO-based GSPO as an example}

Recall Appendix~\ref{sec:gvpo}. We first rewrote the GSPO clipped surrogate objective exactly as "first apply sign-aware clipping to the weight, then multiply by the advantage":
\begin{equation}
	\mathcal{J}_\text{GSPO}(\theta)
	=
	\mathbb{E}\!\left[
	\frac{1}{G}\sum_{i=1}^G
	\widehat{A}_i\,\bar{s}_i(\theta)
	\right],
	\label{eq:minrep_gspo_equiv}
\end{equation}
where $\bar{s}_i(\theta)$ is the post-clipping weight defined in \eqref{eq:postclip_bar_s_piecewise}, and by definition $\bar{s}_i(\theta)>0$.

In DFPO, Min-Replace acts on the post-clipping effective modulation at shared positions, $\bar{\mathbf W}_{t^\star}=\boldsymbol{\alpha}_{t^\star}\odot\bar{\mathbf{s}}$, as follows:
\begin{equation}
	\bar W_{\min,t^\star}(\theta)\triangleq \min_{j\in\{1,\dots,G\}}\bar W_{j,t^\star}(\theta),
	\qquad
	\widetilde W_{i,t^\star}(\theta)\triangleq \bar W_{\min,t^\star}(\theta)\ \ \forall i,
	\label{eq:minrep_def_on_bar}
\end{equation}
and stop-gradient is applied to $\widetilde W_{i,t^\star}(\theta)$ during backpropagation. That is, $\widetilde W_{i,t^\star}$ is treated as a constant coefficient within the group, and the transformation operator itself is not differentiated; see Appendix~\ref{app:nonzero-grad}. If an implementation needs to map back to sequence-level weights and $\alpha_{i,t^\star}>0$, one can set $\widetilde{s}_i=\widetilde W_{i,t^\star}/\alpha_{i,t^\star}$. The bias formulas below are written for position-wise application. If only a single shared position is handled, set $t=t^\star$.

\vspace{0.5em}
\subsection{Gradient form: Min-Replace does not reverse updates, it only applies conservative shrinkage}

Consider a linear segment and ignore subgradient details at breakpoints, which do not affect the sign conclusion. Under stop-gradient, the DFPO gradient can be written as, structurally matching \eqref{eq:gvpo_like_grad},
\begin{equation}
	\nabla_{\theta} \widetilde{\mathcal{J}} (\theta)
	=
	\mathbb{E}\!\left[
	\frac{1}{G} \sum_{i=1}^{G}\sum_{t=1}^{|y_i|}
	\widehat{A}_{i}\,\widetilde W_{i,t}\,
	\nabla_{\theta} \log \pi_{\theta} (y_{i,t} \mid x, y_{i,<t})
	\right],
	\label{eq:minrep_grad_form}
\end{equation}
where for length-averaged GSPO, $\widetilde W_{i,t}=\alpha_{i,t}\widetilde{s}_i$ and $\alpha_{i,t}=1/|y_i|$.

The baseline post-clipping objective \eqref{eq:minrep_gspo_equiv}, without Min-Replace, has the following gradient under the stop-gradient assumption:
\begin{equation}
	\nabla_{\theta} {\mathcal{J}} (\theta)
	=
	\mathbb{E}\!\left[
	\frac{1}{G} \sum_{i=1}^{G}\sum_{t=1}^{|y_i|}
	\widehat{A}_{i}\,\bar W_{i,t}(\theta)\,
	\nabla_{\theta} \log \pi_{\theta} (y_{i,t} \mid x, y_{i,<t})
	\right].
	\label{eq:minrep_grad_base_form}
\end{equation}

For any transformed position $t$, we have $\widetilde W_{i,t}(\theta) = \bar W_{\min,t}(\theta) \leq \bar W_{i,t}(\theta)$, and both quantities are positive. We define a position-wise shrinkage ratio:
\begin{equation}
	\phi_{i,t}(\theta)\triangleq \frac{\widetilde W_{i,t}(\theta)}{\bar W_{i,t}(\theta)}
	=
	\frac{\bar W_{\min,t}(\theta)}{\bar W_{i,t}(\theta)}
	\in (0,1].
	\label{eq:minrep_phi_def}
\end{equation}
Thus, for any trajectory $i$, its effective modulation coefficient satisfies
\begin{equation}
	\widehat{A}_i\,\widetilde W_{i,t}(\theta)
	=
	\phi_{i,t}(\theta)\cdot
	\widehat{A}_i\,\bar W_{i,t}(\theta),
	\qquad \phi_{i,t}(\theta) \in (0,1].
	\label{eq:minrep_coef_scale}
\end{equation}

\begin{theorem}[No reverse update: Min-Replace only applies proportional shrinkage]
	\label{thm:no_reverse_update}
	Under the stop-gradient assumption, for any $i$ and any token $t$, the coefficient of the trajectory's score-function term $\nabla_{\theta}\log\pi_\theta(y_{i,t}\mid x,y_{i,<t})$ after Min-Replace has the same sign as the baseline:
	\begin{equation}
		\mathrm{sign}\!\left(\widehat{A}_i\,\widetilde W_{i,t}(\theta)\right)
		=
		\mathrm{sign}\!\left(\widehat{A}_i\,\bar W_{i,t}(\theta)\right)
		=
		\mathrm{sign}\!\left(\widehat{A}_i\right).
		\label{eq:minrep_sign_preserve}
	\end{equation}
	Therefore, Min-Replace does not cause a reverse update, meaning that the sign of the update determined by the individual trajectory advantage is not flipped. Its effect is equivalent to applying a more conservative effective step size to trajectories whose effective modulation is not minimal.
\end{theorem}

\paragraph{Proof, simplified.}
Since $\bar W_{i,t}(\theta)>0$ and $\widetilde W_{i,t}(\theta)>0$, Eq.~\eqref{eq:minrep_coef_scale} shows that $\widehat{A}_i\,\widetilde W_{i,t}(\theta)$ is a positive scalar multiple of $\widehat{A}_i\,\bar W_{i,t}(\theta)$. The sign is therefore preserved. \qed

\vspace{0.5em}
\subsection{Where the bias comes from: Min-Replace breaks IS unbiasedness and changes the optimization objective}

Although Min-Replace does not create reverse updates, it \textbf{does introduce bias}: it no longer corresponds to the Radon-Nikodym derivative form of the original IS weights, and therefore generally does not satisfy unbiased transport to the target distribution.

This is easiest to see from the perspective of a stop-gradient gradient estimator. Let $\mathbf g_{i,t}(\theta)=\nabla_{\theta} \log \pi_{\theta} (y_{i,t} \mid x, y_{i,<t})$. The single-step gradient estimators of the baseline and Min-Replace are
\begin{equation}
	\widehat{g}_{\mathrm{base}}
	=
	\frac{1}{G}\sum_{i=1}^{G}\sum_{t=1}^{|y_i|}
	\widehat{A}_i\,\bar W_{i,t}(\theta)\,\mathbf g_{i,t}(\theta),
	\qquad
	\widehat{g}_{\mathrm{min}}
	=
	\frac{1}{G}\sum_{i=1}^{G}\sum_{t=1}^{|y_i|}
	\widehat{A}_i\,\bar W_{\min,t}(\theta)\,\mathbf g_{i,t}(\theta).
	\label{eq:minrep_grad_estimators}
\end{equation}
The difference between their expectations, or the bias vector, is
\begin{equation}
	\mathrm{Bias}(\theta)
	\triangleq
	\mathbb{E}\!\left[\widehat{g}_{\mathrm{min}}\right]
	-
	\mathbb{E}\!\left[\widehat{g}_{\mathrm{base}}\right]
	=
	\mathbb{E}\!\left[
	\frac{1}{G}\sum_{i=1}^{G}
	\sum_{t=1}^{|y_i|}
	\widehat{A}_i\big(\bar W_{\min,t}(\theta)-\bar W_{i,t}(\theta)\big)\mathbf g_{i,t}(\theta)
	\right].
	\label{eq:minrep_bias_def}
\end{equation}
Because $\bar W_{\min,t} - \bar W_{i,t} \leq 0$, this bias is generally nonzero, meaning that \textbf{Min-Replace optimizes a more conservative surrogate objective closer to the old policy, not the original IS objective.}

We can also give an upper bound requiring no additional assumptions, showing how the bias increases with the within-group dispersion of effective modulation:
\begin{equation}
	\big\|\mathrm{Bias}(\theta)\big\|
	\le
	\mathbb{E}\!\left[
	\frac{1}{G}\sum_{i=1}^{G}
	\sum_{t=1}^{|y_i|}
	|\widehat{A}_i|\,
	\big(\bar W_{i,t}(\theta)-\bar W_{\min,t}(\theta)\big)\,
	\big\|\mathbf g_{i,t}(\theta)\big\|
	\right].
	\label{eq:minrep_bias_upper}
\end{equation}
Thus, the more dispersed the $\bar W_{i,t}$ values are within the group, especially with long-tailed large ratios or length-scaling differences, the larger the Min-Replace bias becomes. When the values are close, as in a trust region, the bias is small.

\vspace{0.5em}
\subsection{Bias-variance trade-off: why Min-Replace is usually more stable and less likely to trigger clipping feedback}

The direct structural effect of Min-Replace is that it changes the within-group modulation coefficients at shared positions from $\{\widehat{A}_i\bar W_{i,t^\star}\}$ to $\{\widehat{A}_i\bar W_{\min,t^\star}\}$, removing the random modulation caused by within-group dispersion of $\bar W_{i,t^\star}$. This brings two typical benefits.

\paragraph{(1) Within-group asymmetric modulation is weakened, matching the mechanism metric in the main text.}
The main text defines
\begin{equation}
	\mathrm{Asym}(t)
	=
	\mathrm{Var}_{i\in\{1,\dots,G\}}
	\!\left(
	W_{i,t}(\theta)\,\widehat{A}_i
	\right)
	\label{eq:minrep_asym_recall}
\end{equation}
to measure the strength of the "hard-to-cancel" effect for shared or similar tokens. Under Min-Replace, if the modulation factor at a shared position is viewed as $\widehat{A}_i W_{i,t^\star}$, then $W_{i,t^\star}$ is forced to be constant within the group, namely $\bar W_{\min,t^\star}$. This substantially reduces variance, making it easier to restore, or approximate, cancellation in the shared/high-frequency token subspace and thereby lowering the learning tax.

\paragraph{(2) Dominance by tail ratios is suppressed, and clipping/KL constraints are triggered less often with fewer false alarms.}
In PPO/GSPO-style objectives, large-ratio samples often push the surrogate objective into the clipping region, frequently triggering constraints and forming a negative feedback loop in which more updates cause more clipping and more clipping leaves less signal. Min-Replace forces the effective modulation at shared positions to be the smallest effective modulation in the group, applying stronger shrinkage to large-ratio samples. This usually reduces excessive clipping/KL-constraint triggers and improves stability.

The cost is the bias shown in \eqref{eq:minrep_bias_def}: the update becomes more conservative, the effective step size is smaller, convergence may slow, and the optimum may shift toward the old policy.

\vspace{0.5em}
\subsection{When the bias can be ignored: a sufficient trust-region condition}

When training remains in a small trust region, so that post-clipping effective modulations at shared positions are close within the group, the bias is controlled. For example, suppose there exists $\delta>0$ such that all trajectories in the same group satisfy, at the shared position $t^\star$,
\begin{equation}
	0\le
	\log \bar W_{i,t^\star}(\theta)
	-
	\log \bar W_{\min,t^\star}(\theta)
	\le \delta.
	\label{eq:minrep_trust_region_log}
\end{equation}
Then $\bar W_{i,t^\star}(\theta)/\bar W_{\min,t^\star}(\theta)\in[1,e^\delta]$, and
\begin{equation}
	0\le
	\bar W_{i,t^\star}(\theta)-\bar W_{\min,t^\star}(\theta)
	\le
	(e^\delta-1)\bar W_{\min,t^\star}(\theta),
	\qquad
	\phi_{i,t^\star}(\theta)=
	\frac{\bar W_{\min,t^\star}}{\bar W_{i,t^\star}}
	\in[e^{-\delta},1].
	\label{eq:minrep_trust_region_bounds}
\end{equation}
Substituting this into the bias upper bound \eqref{eq:minrep_bias_upper} shows that when $\delta$ is sufficiently small, meaning the relative dispersion of effective modulation is tightly constrained, the upper bound on the Min-Replace bias tends to zero as $\delta\to 0$. In this case, the main effect of Min-Replace is well approximated by variance reduction plus an implicitly smaller step size or stronger trust region, rather than by a severe rewrite of the objective.

\vspace{0.5em}
\subsection{Implications for the ablations and experiments: which effects reflect bias and which reflect structural repair?}

Combining the above with the ablation results in Table~\ref{tab:ablation}:
\begin{itemize}
	\item \textbf{Removing stop-gradient, DFPO no stop-grad, degrades performance:} This is not caused by IS bias itself. Rather, allowing the group transformation to participate in backpropagation introduces additional gradient coupling and instability, breaking the implementation assumption required to treat the transformation as a within-group control variable, and thereby introducing new exchangeability violations.
	\item \textbf{Replacing within-group normalization by global scaling, DFPO scale by 0.5, remains substantially worse:} This shows that the benefit is not merely due to a smaller effective step size or more conservative updates, which would be the bias-level explanation. The key is that Min-Replace removes the \emph{within-group} dispersion of effective modulation, reduces $\mathrm{Asym}$, restores or approximates cancellation in the shared/high-frequency token subspace, and thereby lowers the learning tax as a structural repair.
\end{itemize}

\vspace{0.25em}
\paragraph{Summary.}
Min-Replace breaks IS unbiasedness and introduces bias. Under the stop-gradient assumption, however, it does not cause reverse updates; instead, it applies proportional shrinkage to trajectories whose effective modulation is not minimal. Its core benefits are variance reduction, weaker within-group asymmetric modulation, and lower learning tax. Its cost is a more conservative surrogate objective and potentially slower convergence. These conclusions align precisely with the mechanism analysis of the exchangeability-cancellation necessary condition in this paper.


\newpage
\input{checklist.tex}

\end{document}

%% file: checklist.tex
\section*{NeurIPS Paper Checklist}

The checklist is designed to encourage best practices for responsible machine learning research, addressing issues of reproducibility, transparency, research ethics, and societal impact. Do not remove the checklist: {\bf The papers not including the checklist will be desk rejected.} The checklist should follow the references and follow the (optional) supplemental material.  The checklist does NOT count towards the page
limit. 

Please read the checklist guidelines carefully for information on how to answer these questions. For each question in the checklist:
\begin{itemize}
    \item You should answer \answerYes{}, \answerNo{}, or \answerNA{}.
    \item \answerNA{} means either that the question is Not Applicable for that particular paper or the relevant information is Not Available.
    \item Please provide a short (1--2 sentence) justification right after your answer (even for \answerNA). 
\end{itemize}

{\bf The checklist answers are an integral part of your paper submission.} They are visible to the reviewers, area chairs, senior area chairs, and ethics reviewers. You will also be asked to include it (after eventual revisions) with the final version of your paper, and its final version will be published with the paper.

The reviewers of your paper will be asked to use the checklist as one of the factors in their evaluation. While \answerYes{} is generally preferable to \answerNo{}, it is perfectly acceptable to answer \answerNo{} provided a proper justification is given (e.g., error bars are not reported because it would be too computationally expensive'' or ``we were unable to find the license for the dataset we used''). In general, answering \answerNo{} or \answerNA{} is not grounds for rejection. While the questions are phrased in a binary way, we acknowledge that the true answer is often more nuanced, so please just use your best judgment and write a justification to elaborate. All supporting evidence can appear either in the main paper or the supplemental material, provided in appendix. If you answer \answerYes{} to a question, in the justification please point to the section(s) where related material for the question can be found.

IMPORTANT, please:
\begin{itemize}
    \item {\bf Delete this instruction block, but keep the section heading ``NeurIPS Paper Checklist"},
    \item  {\bf Keep the checklist subsection headings, questions/answers and guidelines below.}
    \item {\bf Do not modify the questions and only use the provided macros for your answers}.
\end{itemize}


\begin{enumerate}

\item {\bf Claims}
    \item[] Question: Do the main claims made in the abstract and introduction accurately reflect the paper's contributions and scope?
    \item[] Answer: \answerYes{}
    \item[] Justification: The abstract and introduction clearly limit the paper's main claim to the necessary condition of token-level gradient exchangeability for intra-group learning under sparse terminal rewards, and summarize the theoretical analysis, structural repair, and empirical validation. The limitations section further states that this condition is necessary but not sufficient, avoiding overgeneralization beyond the theoretical and experimental settings studied in the paper.

\item {\bf Limitations}
    \item[] Question: Does the paper discuss the limitations of the work performed by the authors?
    \item[] Answer: \answerYes{}
    \item[] Justification: The paper includes a dedicated limitations section, explaining that the result is a structural necessary condition rather than a sufficient condition, and discussing limitations such as unidentifiable credit assignment under terminal rewards, interactions between decomposable and coupled mechanisms, and potential bias introduced by projection-based transformations.

\item {\bf Theory assumptions and proofs}
    \item[] Question: For each theoretical result, does the paper provide the full set of assumptions and a complete (and correct) proof?
    \item[] Answer: \answerYes{}
    \item[] Justification: The main text gives definitions, propositions, and corollaries, and explicitly states key assumptions such as shared context-token events, zero-mean group advantages, local linear regions, and stop-gradient treatment. The main proofs and supplementary derivations are provided in the appendix, with proof sketches and cross-references in the main text.

    \item {\bf Experimental result reproducibility}
    \item[] Question: Does the paper fully disclose all the information needed to reproduce the main experimental results of the paper to the extent that it affects the main claims and/or conclusions of the paper (regardless of whether the code and data are provided or not)?
    \item[] Answer: \answerYes{}
    \item[] Justification: The experiments section and appendix specify the tasks, models, baselines, compute-matching protocol, training hyperparameters, inference settings, number of random seeds, and statistical testing procedure. The implementation will be shared through an anonymized public repository.

\item {\bf Open access to data and code}
    \item[] Question: Does the paper provide open access to the data and code, with sufficient instructions to faithfully reproduce the main experimental results, as described in supplemental material?
    \item[] Answer: \answerNo{}
    \item[] Justification: The paper states that the implementation will be shared through a public repository such as GitHub, and the experimental datasets are public benchmarks. However, the current manuscript does not yet provide an accessible anonymized code link, complete run commands, or a script list, so we conservatively answer \answerNo{}.

\item {\bf Experimental setting/details}
    \item[] Question: Does the paper specify all the training and test details (e.g., data splits, hyperparameters, how they were chosen, type of optimizer) necessary to understand the results?
    \item[] Answer: \answerYes{}
    \item[] Justification: The experiments section lists the tasks, models, baselines, and compute-matching protocol, while the appendix provides model context lengths, vLLM version, hardware, learning rate, schedule, warmup, entropy coefficient, rollout count, mini-batch size, and inference decoding settings.

\item {\bf Experiment statistical significance}
    \item[] Question: Does the paper report error bars suitably and correctly defined or other appropriate information about the statistical significance of the experiments?
    \item[] Answer: \answerYes{}
    \item[] Justification: The main results table reports means over five random seeds together with 95\% bootstrap confidence intervals, and states that improvements over baselines are statistically significant under a paired bootstrap test ($p<0.01$).

\item {\bf Experiments compute resources}
    \item[] Question: For each experiment, does the paper provide sufficient information on the computer resources (type of compute workers, memory, time of execution) needed to reproduce the experiments?
    \item[] Answer: \answerYes{}
    \item[] Justification: The appendix reports that the experiments used 32 NVIDIA A800 (80GB) GPUs and that all methods used the same hardware and compute-matching protocol; it also provides detailed hyperparameters.
    
\item {\bf Code of ethics}
    \item[] Question: Does the research conducted in the paper conform, in every respect, with the NeurIPS Code of Ethics \url{https://neurips.cc/public/EthicsGuidelines}?
    \item[] Answer: \answerYes{}
    \item[] Justification: All experiments use publicly available benchmark datasets, involve no personal or sensitive information, and present no obvious ethical risks within the current experimental setting and research scope.

\item {\bf Broader impacts}
    \item[] Question: Does the paper discuss both potential positive societal impacts and negative societal impacts of the work performed?
    \item[] Answer: \answerYes{}
    \item[] Justification: The paper includes a dedicated broader impact section, explaining that GRPO and related group-relative reinforcement learning variants have become an important technical route for post-training reasoning large language models and are adopted by many large-model development teams; therefore, reducing learning tax and improving sample efficiency and training stability can reduce expensive online sampling, GPU time, and iteration cost, giving the method clear economic and engineering value. No negative impact has been identified.
    
\item {\bf Safeguards}
    \item[] Question: Does the paper describe safeguards that have been put in place for responsible release of data or models that have a high risk for misuse (e.g., pre-trained language models, image generators, or scraped datasets)?
    \item[] Answer: \answerNA{}
    \item[] Justification: This paper does not release a new pretrained language model, image generator, or scraped dataset. The experiments are algorithmic evaluations based on public benchmarks and existing base models, so this item is not applicable.

\item {\bf Licenses for existing assets}
    \item[] Question: Are the creators or original owners of assets (e.g., code, data, models), used in the paper, properly credited and are the license and terms of use explicitly mentioned and properly respected?
    \item[] Answer: \answerYes{}
    \item[] Justification: The paper cites the public datasets, models, and inference framework used, and lists their licenses: Qwen3-32B and Qwen3-Next-80B-A3B-Thinking are Apache-2.0 (\url{https://huggingface.co/Qwen/Qwen3-32B}, \url{https://huggingface.co/Qwen/Qwen3-Next-80B-A3B-Thinking}); GPT-OSS-120B is Apache-2.0 and follows the OpenAI gpt-oss usage policy (\url{https://huggingface.co/openai/gpt-oss-120b}, \url{https://openai.com/index/gpt-oss-model-card/}); the MathArena AIME25 and HMMT25 datasets are CC-BY-NC-SA-4.0 (\url{https://huggingface.co/datasets/MathArena/aime_2025}, \url{https://huggingface.co/datasets/MathArena/hmmt_feb_2025}); the official LiveCodeBench repository is MIT, and the Hugging Face code\_generation\_lite data card is marked as Creative Commons/CC while preserving source-platform information such as LeetCode, AtCoder, and Codeforces (\url{https://github.com/LiveCodeBench/LiveCodeBench}, \url{https://huggingface.co/datasets/livecodebench/code_generation_lite}); vLLM 0.11.2 is Apache-2.0 (\url{https://github.com/vllm-project/vllm}). We only use these public assets for non-commercial academic evaluation and inference, do not redistribute the above models, data, or code assets, and follow the corresponding attribution, non-commercial, and usage terms.

\item {\bf New assets}
    \item[] Question: Are new assets introduced in the paper well documented and is the documentation provided alongside the assets?
    \item[] Answer: \answerNA{}
    \item[] Justification: This paper does not introduce or release a new dataset, model, or benchmark asset; its main contributions are theoretical analysis and algorithmic transformations. If implementation code is released later, the anonymized public repository will include reproduction instructions.

\item {\bf Crowdsourcing and research with human subjects}
    \item[] Question: For crowdsourcing experiments and research with human subjects, does the paper include the full text of instructions given to participants and screenshots, if applicable, as well as details about compensation (if any)? 
    \item[] Answer: \answerNA{}
    \item[] Justification: This paper does not include crowdsourcing experiments or human-subject research. All experiments are based on publicly available reasoning benchmarks.

\item {\bf Institutional review board (IRB) approvals or equivalent for research with human subjects}
    \item[] Question: Does the paper describe potential risks incurred by study participants, whether such risks were disclosed to the subjects, and whether Institutional Review Board (IRB) approvals (or an equivalent approval/review based on the requirements of your country or institution) were obtained?
    \item[] Answer: \answerNA{}
    \item[] Justification: This paper does not involve crowdsourcing or human-subject research, so IRB or equivalent review approval is not required.

\item {\bf Declaration of LLM usage}
    \item[] Question: Does the paper describe the usage of LLMs if it is an important, original, or non-standard component of the core methods in this research? Note that if the LLM is used only for writing, editing, or formatting purposes and does \emph{not} impact the core methodology, scientific rigor, or originality of the research, declaration is not required.
    \item[] Answer: \answerYes{}
    \item[] Justification: The research subject and experimental objects of the paper are reinforcement-learning fine-tuning of large language models. The experiments section and appendix explicitly state the use of Qwen3-32B, Qwen3-Next-80B-A3B-Thinking, GPT-OSS-120B, as well as context lengths, vLLM version, and decoding settings.

\end{enumerate}